\documentclass[preprint,12pt]{colt2026} 

\title[Tight Long-Term Tail Decay of (Clipped) SGD in Non-Convex Optimization]{Tight Long-Term Tail Decay of (Clipped) SGD in \\ Non-Convex Optimization}
\usepackage{times}
\usepackage{nicefrac}
\usepackage{microtype}
\usepackage{graphicx}
\usepackage{subcaption}
\usepackage{booktabs} 

\usepackage{amsmath}
\usepackage{amssymb}
\usepackage{mathtools}
\usepackage{comment}
\usepackage{dsfont}
\usepackage{gensymb}
\usepackage[shortlabels]{enumitem}
\usepackage{nicefrac}
\usepackage{threeparttable}
\usepackage{changepage}
\usepackage{multirow}
\usepackage{caption}
\usepackage{algorithm}
\usepackage{algorithmic}
\newtheorem{assumption}{Assumption}

\coltauthor{%
 \Name{Aleksandar Armacki} \Email{aleksandar.armacki@epfl.ch}\\
 \addr {\'E}cole Polytechnique F{\'e}d{\'e}rale de Lausanne, Lausanne, Switzerland
 \AND
 \Name{Dragana Bajovi{\'c}} \Email{dbajovic@uns.ac.rs}\\
 \addr Faculty of Technical Sciences, University of Novi Sad, Novi Sad, Serbia
 \AND
 \Name{Du{\v s}an Jakoveti{\'c}} \Email{dusan.jakovetic@dmi.uns.ac.rs}\\
 \addr Faculty of Sciences, University of Novi Sad, Novi Sad, Serbia
 \AND
 \Name{Soummya Kar} \Email{soummyak@andrew.cmu.edu}\\
 \addr Carnegie Mellon University, Pittsburgh, Pennsylvania, USA
 \AND
 \Name{{Ali H.} Sayed} \Email{ali.sayed@epfl.ch}\\
 \addr {\'E}cole Polytechnique F{\'e}d{\'e}rale de Lausanne, Lausanne, Switzerland%
}

\usepackage{xspace}

\newcommand{\lp}{\left(}
\newcommand{\rp}{\right)}

\newcommand{\lbr}{\left[}
\newcommand{\rbr}{\right]}

\usepackage{xcolor}

\usepackage{tikz}
\usetikzlibrary{calc,trees,positioning,arrows,chains,shapes.geometric,%
	decorations.pathreplacing,decorations.pathmorphing,shapes,%
	matrix,shapes.symbols}

\tikzstyle{startstop} = [rectangle, draw, rounded corners, align=center, minimum width=3cm, minimum height=1cm,text centered]
\tikzstyle{decision} = [diamond, draw, fill=blue!20, 
text width=4.5em, text badly centered, node distance=3cm, inner sep=0pt]
\tikzstyle{block} = [rectangle, draw, fill=blue!10, align=center, rounded corners, minimum width=3cm, minimum height=1cm]
\tikzstyle{blockcast} = [rectangle, draw, fill=red!10, align=center, rounded corners, minimum width=3cm, minimum height=0.45cm]
\tikzstyle{line} = [draw, -latex']
\tikzstyle{cloud} = [draw, ellipse,fill=red!20, node distance=3cm,
minimum height=2em]

\newcommand{\R}{\mathbb{R}}
\newcommand{\E}{\mathbb{E}}
\newcommand{\N}{\mathbb{N}}
\newcommand{\D}{\mathcal{D}}

\newcommand{\Prob}{\mathbb{P}}
\newcommand{\sgd}{\textup{\textbf{\texttt{SGD}}}\xspace}
\newcommand{\csgd}{\textup{\textbf{\texttt{c-SGD}}}\xspace}
\newcommand{\sfo}{$\mathcal{SFO}$\xspace}

\newcommand{\bigO}{\mathcal{O}}

\begin{document}

\maketitle

\begin{abstract}%
  The study of tail behaviour of \sgd-induced processes has been attracting a lot of interest, due to offering strong guarantees with respect to individual runs of an algorithm. While many works provide high-probability guarantees, quantifying the error rate for a fixed probability threshold, there is a lack of work directly studying the probability of failure, i.e., quantifying the tail decay rate for a fixed error threshold. Moreover, existing results are of finite-time nature, limiting their ability to capture the true long-term tail decay which is more informative for modern learning models, typically trained for millions of iterations. Our work closes these gaps, by studying the long-term tail decay of \sgd-based methods through the lens of large deviations theory, establishing several strong results in the process. First, we provide an upper bound on the tails of the gradient norm-squared of the best iterate produced by (vanilla) \sgd, for non-convex costs and bounded noise, with long-term decay at rate $e^{-\frac{t}{\log(t)}}$. Next, we relax the noise assumption by considering clipped \sgd (\csgd) under heavy-tailed noise with bounded moment of order $p \in (1,2]$, showing an upper bound with long-term decay at rate $e^{-\frac{t^{\beta_p}}{\log(t)}}$, where $\beta_p = \frac{4(p-1)}{3p-2}$ for $p \in (1,2)$ and $e^{-\frac{t}{\log^2(t)}}$ for $p = 2$. Finally, we provide lower bounds on the tail decay, at rate $e^{-t}$, showing that our rates for both \sgd and \csgd are tight, up to \linebreak poly-logarithmic factors. Notably, our results demonstrate \emph{an order of magnitude faster} long-term tail \linebreak decay compared to existing work based on finite-time bounds, which show rates $e^{-\sqrt{t}}$ and $e^{-t^{\nicefrac{\beta_p}{2}}}$, $p \in (1,2]$, for \sgd and \csgd, respectively. As such, we uncover regimes where the tails decay \linebreak much faster than previously known, providing stronger long-term guarantees for individual runs.
\end{abstract}


\section{Introduction}

Non-convex optimization is an integral part of modern machine learning (ML), as many practical settings, such as training large language models (LLMs), e.g., \cite{zhang2022opt}, represent instances of the general problem of minimizing a non-convex cost $f: \R^d \mapsto \R$, given by
\begin{equation}\label{eq:problem}
    \min_{x \in \R^d} f(x).
\end{equation} In practice, the problem \eqref{eq:problem} is solved using iterative methods, typically based on the stochastic gradient descent (\sgd) algorithm \cite{robbins1951stochastic}, making the study of convergence guarantees of \sgd-based methods an integral part of ML theory. While classical guarantees, like mean-squared error (MSE) convergence, e.g., \cite{Sayed_inference}, offer important indicators of average performance across many runs, the advent of large-scale models, such as transformers and LLMs, has shifted the focus on guarantees with respect to an \emph{individual run} of an algorithm, as even a single training run of such models can be incredibly costly, both time- and resource-wise. This has led to an increased interest in \emph{tail probabilities} and results of type $\Prob\big(F_t > \epsilon_t\big) \leq \delta_t$, where $F_t \coloneqq \min_{k \in [t]}\|\nabla f(x_k)\|^2$ is the standard metric of interest for non-convex costs, $\{x_t\}_{t \in \N}$ is the sequence of models generated by an iterative method, while $\epsilon_t > 0$ and $\delta_t \in (0,1]$ are the \emph{error} and \emph{probability thresholds}, respectively. Conventional results quantify the tail behaviour for a \emph{fixed probability threshold} and \emph{decaying error threshold}, i.e., the goal is to show $\Prob\big(F_t > c(\delta)/n_t\big) \leq \delta$, for any fixed $\delta \in (0,1)$, where $n_t \rightarrow \infty$ is the \emph{decay rate} and $c(\delta)$ is a function of the probability threshold. Of particular interest are high-probability (HP) convergence results, where $c(\delta) = \log(\nicefrac{1}{\delta})$, e.g., \cite{nemirovski2009robust,ghadimi2013stochastic}, which can be used to establish sharp \emph{tail bounds}, i.e., \emph{exponentially decaying probability threshold} for a \emph{fixed error threshold}, of the form $\Prob(F_t > \epsilon) \leq e^{-n_t\epsilon}$, for any fixed $\epsilon > 0$ (obtained from HP bounds by solving $\frac{\log(\nicefrac{1}{\delta})}{n_t} = \epsilon$ for $\delta$). Tail bounds for fixed error thresholds are important, as in practice the goal is to reach a \emph{$\epsilon$-stationary point}, i.e., a point $x \in \R^d$ that satisfies $\|\nabla f(x)\| \leq \epsilon$, e.g., \cite{arjevani2022}, hence tail bounds quantify the \emph{probability of failure} (i.e., not reaching an $\epsilon$-stationary point)\footnote{It can be easily seen that if $\min_{k \in [t]}\|\nabla f(x_k)\| > \epsilon$, then an $\epsilon$-stationary point has not been reached up to time $t$.} for an individual run. 

Another important observation is that classical tail bounds are of \emph{finite-time} nature, designed to hold \emph{for every $t \in \N$}. While useful for smaller models, trained for a moderate number of iterations (e.g., $t = 10^4$-$10^5$), such results might not be too informative for modern large-scale models that typically require a huge number of training iterations, e.g., \texttt{AlexNet} and \texttt{ResNet-50} are both trained for approximately $t = 5\times10^5$ iterations \cite{krizhevsky-alexnet,He_2016_CVPR}, the original BERT model is trained for $t = 10^6$ iterations \cite{devlin-etal-2019-bert}, while LLMs are estimated to require many millions of training iterations, see, e.g., \cite{NEURIPS2022_c1e2faff}. As such, bounds capturing the long-term tail behaviour, which holds \emph{for $t$ sufficiently large}, are more appropriate for models requiring an enormous number of training iterations. Although finite-time results derived for a fixed probability threshold, like HP guarantees, can imply exponentially decaying tail bounds, as we show later, they do so in a conservative manner, resulting in very loose bounds in the regime of large $t$. Motivated by these observations, we consider the following important question:

\begin{center}
\emph{What is the sharpest long-term tail decay achievable for \sgd-type methods in \\ non-convex optimization, for a given, fixed error threshold}?    
\end{center}

To answer this question, we take a large deviations principle (LDP) approach, e.g., \cite{dembo2009large}, by directly studying the long-term tail behaviour of $F_t$, for a fixed error threshold. In Table \ref{tab:comp} we provide an overview of long-term tail decay rates for \sgd-type methods, where it can be seen that our results show \emph{an order of magnitude faster} long-term tail decay rate of $F_t$ than implied by prior results based on finite-time bounds, indicating that a strictly sharper tail decay, not properly captured by existing works, is achievable in the long run. We next state our full contributions.

\begin{table*}[t]
\caption{\scriptsize Long-term tail decay of \sgd-based methods in non-convex optimization. \emph{Method} specifies the variant of \sgd; \emph{Cost} states the assumptions on the cost function; \emph{Noise} provides the noise assumptions; \emph{Decay rate} is the largest positive sequence $n_t \rightarrow \infty$ such that, for any $\epsilon > 0$ and some $C_\epsilon > 0$, we have $\limsup_{t \rightarrow \infty}n_t^{-1}\log\mathbb{P}(\min_{k \in [t]}\|\nabla f(x_k)\|^2 > \epsilon) \leq -C_\epsilon$.$^\S$ The decay rates for \cite{liu2023high,cutkosky2021high,nguyen2023improved} stem from their finite-time HP bounds, whereas the decay rates from \cite{armacki2026sharp,armacki2024_ldp+mse}$^\dagger$ and our work are obtained by establishing an LDP upper bound with a full rate function, see Section \ref{subsec:large-dev-primer} for details.}
\label{tab:comp}
\begin{adjustwidth}{-1in}{-1in} 
\begin{center}
\begin{threeparttable}
\begin{small}
\begin{sc}
\begin{tabular}{ccccc}
\toprule
\multicolumn{1}{c}{\scriptsize Work} & \multicolumn{1}{c}{\scriptsize Method} & \multicolumn{1}{c}{\scriptsize Cost} & \multicolumn{1}{c}{\scriptsize Noise} & \multicolumn{1}{c}{\scriptsize Decay rate} \\
\midrule
\scriptsize\citet{liu2023high} & \scriptsize \sgd & $\substack{\text{\scriptsize{smooth, bounded}} \\ \text{\scriptsize{from below}}}$ & \scriptsize sub-Gaussian & \scriptsize $\sqrt{t}$ \\ 
\midrule
\scriptsize This work (Theorem \ref{thm:main-non-conv}) & \scriptsize \sgd & $\substack{\text{\scriptsize{smooth, bounded from}} \\ \text{\scriptsize{below, bounded gradients}}}$ & \scriptsize a.s. bounded  & \scriptsize $t/\log(t)$ \\
\midrule
\scriptsize\citet{cutkosky2021high} & $\substack{\text{\scriptsize \texttt{\textbf{Normalized}}} \\ \text{\scriptsize \csgd}}$ & $\substack{\text{\scriptsize{smooth, bounded from}} \\ \text{\scriptsize{below, bounded gradients}}}$ & $\substack{\text{\scriptsize{bounded moment}} \\ \text{\scriptsize{of order $p \in (1,2]$}}}$ & \scriptsize$t^{\frac{2(p-1)}{3p-2}}$ \\ 
\midrule
\scriptsize\citet{nguyen2023improved} & \scriptsize \csgd & $\substack{\text{\scriptsize{smooth, bounded}} \\ \text{\scriptsize{from below}}}$ & $\substack{\text{\scriptsize{bounded moment}} \\ \text{\scriptsize{of order $p \in (1,2]$}}}$ & \scriptsize$t^{\frac{2(p-1)}{3p-2}}/\log^{\frac{2p}{3p-2}}(t)$ \\ 
\midrule
\scriptsize\citet{armacki2026sharp,armacki2024_ldp+mse} & \scriptsize \texttt{\textbf{N-SGD}}$^\ddagger$ & $\substack{\text{\scriptsize{smooth, bounded}} \\ \text{\scriptsize{from below}}}$ & $\substack{\text{\scriptsize{symmetric pdf,}} \\ \text{\scriptsize{positive around origin}}}$ & \scriptsize $\sqrt{t}/\log(t)$ \\ 
\midrule
\scriptsize This work (Theorem \ref{thm:main-non-conv-clip}) & \scriptsize \csgd & $\substack{\text{\scriptsize{smooth, bounded from}} \\ \text{\scriptsize{below, bounded gradients}}}$ & $\substack{\text{\scriptsize{bounded moment}} \\ \text{\scriptsize{of order $p \in (1,2]$}}}$  & \scriptsize $t^{\frac{4(p-1)}{3p-2}}/\log(t)^\P$ \\
\bottomrule
\end{tabular}
\end{sc}
\end{small}
\begin{tablenotes}\scriptsize
    \item[$\S$] Although some works included in the table provide bounds on different quantities (e.g., $\frac{1}{t}\sum_{k = 1}^t\|\nabla f(x_k)\|^2$), all of them imply a bound on $\min_{k \in [t]}\|\nabla f(x_k)\|^2$. As discussed in Section \ref{sec:main} and Appendix \ref{app:metric}, our results can be equivalently stated in terms of $\frac{1}{t}\sum_{k = 1}^t\|\nabla f(x_k)\|^2$.
    \item[$\dagger$] While \citet{armacki2026sharp} provide HP bounds for \texttt{\textbf{N-SGD}}, their results can be used to get LDP upper bounds, see \cite{armacki2025thesis} for details. 
    \item[$\ddagger$] \texttt{\textbf{N-SGD}} is a general nonlinear \sgd framework which, for state-dependent noise, among others includes clipping, normalization, smooth sign and smooth component-wise clipping. If the noise is also independent, identically distributed (IID), in addition to the previous, the \texttt{\textbf{N-SGD}} framework includes non-smooth component-wise nonlinearities, like standard sign and component-wise clipping. 
    \item[$\P$] For the special case $p = 2$, our decay rate incurs an additional $\log$ factor, i.e., is of the form $n_t = t/\log^2(t)$, see Theorem \ref{thm:main-non-conv-clip} for details.
\end{tablenotes}
\end{threeparttable}
\end{center}
\vskip -0.1in
\end{adjustwidth}
\end{table*}

\subsection{Contributions}

Motivated by the need for sharp bounds on the failure probability of individual runs in applications like training modern large-scale models, we study the long-term tail behaviour of iterates induced by \sgd-type methods through the lens of large deviations (LD) theory. Our contributions are as follows. 

\begin{itemize}[leftmargin=*]
    \item We provide a sharp characterization of the failure probability of \sgd-type methods in non-convex optimization, by studying the long-term tail behaviour for a fixed error threshold. First, we consider vanilla \sgd under almost surely (a.s.) bounded noise and establish an LDP upper bound on $F_t$, with an exponential long-term tail decay, at rate $n_t = t/\log(t)$. In Table \ref{tab:comp} we summarize existing long-term tail decay results. We can see that the decay rate for \sgd in our work is an order of magnitude faster than the $n_t = \sqrt{t}$ rate resulting from finite-time HP result in \cite{liu2023high}. 

    \item Next, we relax the noise condition, by considering clipped \sgd (\csgd) under heavy-tailed noise \linebreak with bounded moment of order $p \in (1,2]$. We provide an LDP upper bound on $F_t$, for an appropri-\linebreak ately chosen clipping threshold, with long-term exponential tail decay at rate $n_t = t/\log^2(t)$ for \linebreak $p = 2$ and $n_t = t^{\frac{4(p-1)}{3p-2}}/\log(t)$ for $p \in (1,2)$. We can again see in Table \ref{tab:comp} that the long-\linebreak term tail decay rate for \csgd in our work is an order of magnitude faster than the $n_t = t^{\frac{2(p-1)}{3p-2}}$ rate resulting from finite-time HP bounds in \cite{cutkosky2021high,nguyen2023improved}.

    \item Finally, we show that the long-term decay rates established in our LDP upper bounds are tight, by providing matching finite-time (and asymptotic) lower bounds on the tail probability induced by both vanilla \sgd and \csgd. To do so, we carefully construct an instance of cost, noise, model initialization and error threshold, under which we show that the tails induced by both methods exhibit exponentially decaying lower bounds at rate $n_t = t$, demonstrating that our long-term upper bounds for \sgd and \csgd (when $p = 2$) are tight, up to poly-logarithmic factors.
    
\end{itemize}

As such, our results show that the long-term tail probability induced by \sgd-type methods in non-convex optimization decays at a rate that is both tight and significantly faster than previously known, leading to much sharper bounds on the probability of failure in non-convex problems (i.e., not reaching an $\epsilon$-stationary point) and stronger guarantees for individual runs of an algorithm.

\subsection{Literature Review}

We next review the literature on finite-time high-probability (HP) and asymptotic LD results for \sgd-based methods. For an overview of other popular types of guarantees, see Appendix \ref{app:literature}.

\paragraph{High-probability guarantees.} Initial HP results consider light-tailed noise and include \cite{nemirovski2009robust,lan2012optimal,ghadimi2013stochastic,hazan2015beyond,harvey2019tight,li2020high,liu2024revisiting}, with \citet{liu2023high} providing HP convergence of $F_t$ for vanilla \sgd and non-convex costs, with order-optimal rate $\bigO\Big(\frac{\log(\nicefrac{t}{\delta})}{\sqrt{t}}\Big)$. More recently, HP convergence of nonlinear \sgd methods (e.g., clipping, sign, normalization) under noise with heavier tails has attracted attention, starting with \citet{gorbunov2020stochastic,parletta2022high}, who consider noise with bounded variance, while \citet{li2022high,eldowa2024general,madden2020high} consider sub-Weibull noise. This is extended by \citet{cutkosky2021high,nguyen2023improved, sadiev2023highprobability,liu2023breaking,hubler2025normalization,kornilov2025sign}, who consider various nonlinear \sgd methods under noise with bounded moment of order $p \in (1,2]$, and \citet{armacki2025high,armacki2026sharp}, who consider a unified nonlinear \sgd framework (dubbed \texttt{\textbf{N-SGD}}) under noise with symmetric probability density function (PDF), positive around zero and potentially unbounded moments. Among them, \citet{cutkosky2021high,nguyen2023improved} show that $F_t$ achieves the optimal rate $\bigO\Big(\log(\nicefrac{1}{\delta})t^{\frac{2(1-p)}{3p-2}} \Big)$ using \csgd,\footnote{Technically, \citet{cutkosky2021high} consider a variant of momentum \sgd, using both clipping and normalization.} while \citet{armacki2026sharp} show that \texttt{\textbf{N-SGD}} converges with rate $\bigO\Big(\frac{\log(\nicefrac{t}{\delta})}{\sqrt{t}}\Big)$, matching the rate in \cite{liu2023high} established under light tails.\footnote{\citet{hubler2025normalization,kornilov2025sign} show that normalized and sign \sgd match the oracle complexity of clipped \sgd, using an increasing batch size. Since we study the long-term behaviour, this implies an infinite batch size as $t \rightarrow \infty$, hence we focus our comparison on works that use a fixed batch size. Next, while the rate in \cite{armacki2026sharp} is better than the one in \cite{cutkosky2021high,nguyen2023improved} for any $p < 2$, it does not invalidate the optimality of their rate, as the two are derived under different conditions on the noise.} Translated into tail bounds, it follows that \cite{liu2023high,armacki2026sharp} imply an asymptotic exponential tail decay for vanilla \sgd under light-tailed noise and general nonlinear \sgd under noise with symmetric PDF, respectively, with decay rate $n_t = \sqrt{t}$. Similarly, \cite{cutkosky2021high,nguyen2023improved} imply a long-term exponential tail decay of $F_t$ for \csgd under bounded $p$-th moment noise, with decay rate $n_t = t^{\frac{2(p-1)}{3p-2}}$.  

\paragraph{Large deviations guarantees.} LD studies have a long history, see, e.g., \cite{varadhan,dembo2009large} and references therein, with a wide range of applications, including statistical mechanics \cite{ellis-ld,TOUCHETTE20091}, distributed detection \cite{bajovic-detection-ld,bajovic-detection-ld2,braca-ld,matta-diffusion-1,matta-difussion-2}, social learning \cite{asl2021,bajovic-inference-ld,social-learning-book} and general ML \cite{braca-ldp-classification,lindhe-large}. In the context of \sgd-type methods, LDs are studied in \cite{hu2019diffusion,pmlr-v206-bajovic23a,pmlr-v242-jongeneel24a,azizian2024long,azizian2025global,gurbuzbalaban2025accelerated,armacki2024_ldp+mse}. \citet{hu2019diffusion,azizian2024long,azizian2025global} use a LD approach to study the behaviour of \sgd iterates with fixed step-size for non-convex problems, in the limit as the step-size goes to zero. \citet{hu2019diffusion} show that iterates can escape local minimizers in a number of iterations exponentially depending on the inverse of the step-size, while \citet{azizian2024long,azizian2025global} show that the iterates concentrate around local minima and establish a full LDP for the time it takes to reach a global minima, also exponentially depending on the inverse of the step-size.\footnote{A full LDP means that matching lower and upper bounds are provided, see Section \ref{subsec:large-dev-primer} for details.} \citet{pmlr-v206-bajovic23a} provide an LDP upper bound for the last iterate of \sgd and strongly convex costs, while \citet{pmlr-v242-jongeneel24a} extend this result to costs satisfying the P\L{} condition, with applications to reinforcement learning. \citet{gurbuzbalaban2025accelerated} study a class of generalized momentum methods for strongly convex costs, establishing an LDP upper bound for the average cost sub-optimality. Finally, \citet{armacki2024_ldp+mse} consider non-convex costs and a general nonlinear \sgd framework dubbed \textbf{\texttt{N-SGD}}, under noise with symmetric PDF, positive around zero and unbounded moments, showing an LDP upper bound for $F_t$, with decay rate $n_t = \sqrt{t}/\log(t)$. As mentioned, \cite{hu2019diffusion,azizian2024long,azizian2025global} study the tail behaviour of iterates of \sgd for non-convex costs using a fixed step-size, in the limit as the step-size goes to zero. On the other hand, we study the tail behaviour of $F_t$ and consider the more natural variant of \sgd-based methods using a time-varying step-size, in the limit as the number of iterations goes to infinity.   

\paragraph{Technical challenges and novelty.} In order to establish our results, we needed to overcome a number of challenges. First, to show LDP-type results and accelerated tail decay rates, we provide tight bounds on the moment-generating function (MGF) of $F_t$ and use the G\"{a}rtner-Ellis theorem (see Proposition \ref{prop:gartner-ellis} in Appendix \ref{app:prelim}). Next, in order to be able to apply the G\"{a}rtner-Ellis theorem, we needed to show that the MGF is finite everywhere, which is significantly stronger than the finite-time HP analysis, where it suffices to show that the MGF is bounded locally, e.g., over a compact interval or even at a single point (see Section \ref{sec:discus} for a more detailed discussion). To resolve this challenge, we consider uniformly bounded noise for \sgd, while for \csgd we show via a careful analysis and tuning of the clipping threshold that the MGF is bounded everywhere, even under general heavy-tailed noise. Finally, we provide a lower bound on the tail probability, by carefully constructing an instance of cost, noise and model initialization for which the tails of $F_t$ decay at least exponentially fast and which is of independent interest in HP-type studies.   

\paragraph{Paper organization.} Section \ref{sec:prelim} provides preliminaries, Section \ref{sec:main} presents the main results, Section \ref{sec:discus} provides comparison with existing works, Section \ref{sec:conc} concludes the paper, while Appendix contains results omitted from the main body. The remainder of this section introduces some notation.

\paragraph{Notation.} We denote by $\N,\R$ and $\R^d$ the sets of positive integers, real numbers and $d$-dimensional real vectors. For any $m \in \N$, we denote by $[m]$ the set $[m] = \{1,2,\ldots,m\}$. The Euclidean inner product and induced norm are denoted by $\langle \cdot , \cdot \rangle$ and $\|\cdot\|$. For a set $B$, we denote by $B^{\degree}$ and $\overline{B}$ its topological interior and closure. We use $o(\cdot)$ and $\bigO(\cdot)$ as the standard ``little o'' and ``big O'', i.e., for two positive sequences $\{a_t\}_{t \in \N},\{b_t\}_{t \in \N}$, such that $\lim_{t \rightarrow \infty}a_t = \lim_{t \rightarrow \infty} b_t = \infty$, we say that $a_t = o(b_t)$ ($a_t = \bigO(b_t)$), if $\lim_{t\rightarrow \infty}\frac{a_t}{b_t} = 0$ ($\limsup_{t \rightarrow \infty}\frac{a_t}{b_t} < \infty$), unless stated otherwise.

\section{Preliminaries}\label{sec:prelim}

In this section we provide the preliminaries. Subsection \ref{subsec:oracle} specifies the oracle model and reviews the vanilla and clipped \sgd methods, while Subsection \ref{subsec:large-dev-primer} provides a primer on LD theory.

\subsection{The Oracle Model and \sgd-based Methods}\label{subsec:oracle}

We assume access to a Stochastic First-order Oracle (\sfo), which, when queried with input $x$, returns a stochastic estimate $g$ of the gradient $\nabla f(x)$. The \sfo subsumes the following paradigms.

\paragraph{1. Batch (i.e., offline) learning:} for a finite dataset $\{\xi^{i}\}_{i \in [m]}$ and loss $\ell: \R^d \times \Xi \mapsto \R$, the cost is given by $f(x) = \frac{1}{m}\sum_{i \in [m]}\ell(x;\xi^i)$. When queried, the \sfo chooses a sample of indices $S \subset [m]$ uniformly at random and outputs $g = \frac{1}{|S|}\sum_{j \in S}\nabla \ell(x;\xi^j)$, where $1 \leq |S| < m$.

\paragraph{2. Streaming (i.e., online) learning:} for a random variable $\xi$ following an unknown distribution $\D$, the cost is given by $f(x) = \E_{\xi \sim \D}\big[ \ell(x;\xi) \big]$. When queried, the \sfo generates a mini-batch $\{\xi^j\}_{j \in S}$ of IID copies of $\xi$ and outputs $g = \frac{1}{|S|}\sum_{j \in S}\nabla \ell(x;\xi^j)$, where $|S| \geq 1$.

\paragraph{}Next, we describe the two methods considered in our work, namely (vanilla) \sgd and \csgd. The general update rule for iterative \sgd-based methods can be represented as
\begin{equation}\label{eq:sgd}
    x_{t+1} = x_t - \alpha_t\Psi_t(g_t),
\end{equation} where $\alpha_t > 0$ is the step-size, while $\Psi_t: \R^d \mapsto \R^d$ is a (possibly) nonlinear mapping. The first method considered in our work, \sgd, where $\Psi_t$ is the linear identity map $\Psi_t(x) \equiv x$, is perhaps the most well-known and widely used algorithm, lauded for its ease of implementation and strong performance. However, the advent of deep learning and LLMs resulted in exponentially increasing model complexity and phenomena such as heavy-tailed noise and exploding gradients, e.g., \linebreak \citet{pmlr-v28-pascanu13,simsekli2019tail,zhang2020adaptive,heavy-tail-phenomena}, requiring nonlinear modifications to \sgd, with clipping being a very popular choice. It is known to bring many benefits, like stabilizing and accelerating training \cite{zhang2020gradient}, ensuring convergence under heavy-tailed noise \cite{sadiev2023highprobability} and providing differential privacy \cite{zhang2022clip_FL_icml}. The resulting method, \csgd, is widely used for LLM training \cite{zhang2022opt,touvron2023llama,liu2024deepseek}, and represents an instance of \eqref{eq:sgd} with $\Psi_t(x) = \min\Big\{1,\frac{\gamma_t}{\|x\|}\Big\}x$, for a user-specified clipping threshold $\gamma_t > 0$. The two methods are summarized in Algorithm \ref{alg:sgd}.

\begin{algorithm}[tb]
\caption{Vanilla and clipped \sgd}
\label{alg:sgd}
{\setlength{\baselineskip}{1.05\baselineskip}
\begin{algorithmic}[1]
   \REQUIRE Model initialization $x_1 \in \R^{d}$, step-size $\{\alpha_t\}_{t \in \N}$, clipping threshold $\{\gamma_t\}_{t \in \N}$;
   \FOR{$t = 1,2,\ldots$}
        \STATE Query the \sfo with input $x_t$ and obtain $g_t$;
        \STATE Perform the following model update:
        \STATE \hspace{0.75em}Option 1:\; 
        $x_{t+1} = x_t - \alpha_t g_t$;
        \hfill{(Vanilla \sgd)}
        \STATE \hspace{0.75em}Option 2:\;
        $x_{t+1} = x_t - \alpha_t
        \min\!\Big\{1,\frac{\gamma_t}{\|g_t\|}\Big\} g_t$;
        \hfill{(Clipped \sgd)}
    \ENDFOR
\end{algorithmic}
}
\end{algorithm}

\subsection{Large Deviations Principle: a Background}\label{subsec:large-dev-primer}

The goal of LD studies is to quantify the long-term probability of (rare) events, by providing sharp, exponentially decaying long-term tail bounds. In particular, for a process of interest $\{F_t\}_{t \in \N}$, the LDP aims to find a lower semi-continuous function $I: \R \mapsto [0,\infty]$ and a positive sequence $\{n_t\}_{t \in \N}$ satisfying $\lim_{t \rightarrow \infty}n_t = \infty$, such that, for any (Borel) measurable $B \subseteq \R$
\begin{equation}\label{eq:ldp}
    -\inf_{x \in B^{\degree}} I(x) \leq \liminf_{t \rightarrow \infty} \frac{1}{n_t}\log\Prob(F_t \in B) \leq \limsup_{t \rightarrow \infty} \frac{1}{n_t}\log \Prob(F_t \in B) \leq -\inf_{x \in \overline{B}} I(x).
\end{equation} If \eqref{eq:ldp} holds, the process $\{F_t\}_{t \in \N}$ is said to satisfy the \emph{(full) LDP}, with decay rate $n_t$ and rate function $I$, see, e.g., \cite{dembo2009large}. The relation \eqref{eq:ldp} provides a tight characterization of the asymptotic, long-term behaviour of $F_t$. While the full LDP is desirable, the lower bound in \eqref{eq:ldp} is often difficult to obtain. Instead, one typically aims to establish (only) the upper bound, in which case $\{F_t\}_{t \in \N}$ is said to satisfy the \emph{LDP upper bound}. Note that the upper bound in \eqref{eq:ldp} implies 
\begin{equation}\label{eq:ldp-equiv}
    \Prob(F_t \in B) \leq e^{-n_t\inf_{x \in \overline{B}}I(x) + o(n_t)} \approx e^{-n_t\inf_{x \in \overline{B}}I(x)},
\end{equation} where the second relation holds for all $t$ sufficiently large, hence the LDP upper bound alone is a very strong indicator of the long-term behaviour of $F_t$, establishing exponentially decaying long-term probability of $F_t$ ending up in any set $B$, such that $\inf_{x \in \overline{B}}I(x) > 0$.

\section{Main Results}\label{sec:main}

In this section we provide the main results. Subsection \ref{subsec:assump} states the assumptions, Subsection \ref{subsec:main-sgd} provides results for \sgd, Subsection \ref{subsec:main-clip} presents results for \csgd under heavy-tailed noise, while Subsection \ref{subsec:sgd-lbd} establishes a (nearly) matching lower bound on the tail probability.

\subsection{Assumptions}\label{subsec:assump}

We start by stating a technical condition on the model initialization, used for analysis purposes. 

\begin{assumption}\label{asmpt:init}
    The model initialization $x_1 \in \R^d$ is selected in a deterministic manner.
\end{assumption}

\paragraph{}Assumption \ref{asmpt:init} allows the initialization to be any real vector, as long as it is a deterministic quantity.

\begin{assumption}\label{asmpt:cost}
    The cost is bounded from below, has uniformly bounded gradients and is $L$-smooth, i.e., it holds that $f_\star \coloneqq \inf_{x \in \R^d}f(x) > -\infty$ and for some $G > 0$ and any $x, y\in \R^d$, we have 
    \begin{equation*}
         \|\nabla f(x)\| \leq G \:\:\text{and}\:\: f(x) \leq f(y) + \langle \nabla f(y), x - y \rangle + \frac{L}{2}\|x - y\|^2.
    \end{equation*}
\end{assumption} 

Assumption \ref{asmpt:cost} specifies conditions on the cost. Boundedness from below and $L$-smoothness are standard for general non-convex costs, e.g., \cite{ghadimi2013stochastic}. The bounded gradient condition is also widely used for non-convex costs, e.g., \cite{NIPS2017_766ebcd5,cevher-almost_sure,cutkosky2021high} and is satisfied, e.g., by any $G$-Lipschitz continuous cost, which includes a wide class of convolutional and deep neural networks, e.g., \cite{fazlyab2019lipschitz,dongmian2020lipschitz,combettes2020lipschitz,zhang2022lipschitz}, as well as transformer-based models, e.g., \cite{pmlr-v139-kim21i}. For ease of notation, let $z_t \coloneqq g_t - \nabla f(x_t)$ denote the gradient noise and let $\mathcal{F}_t \coloneqq \sigma\left(\{x_1,\ldots,x_t\}\right)$ be the natural filtration, with $\mathcal{F}_1 \coloneqq \sigma\left(\{\emptyset, \Omega\}\right)$ being the trivial $\sigma$-algebra. The next two assumptions state the noise regimes considered in our work.  
\begin{assumption}\label{asmpt:noise-bounded}
    The gradient estimator is unbiased and the noise is uniformly bounded, i.e., for all $t \geq 1$ and some $M > 0$, it holds that $\E\big[g_t\: \vert \: \mathcal{F}_{t}\big] = \nabla f(x_t)$ and $\|z_t\| \leq M$, a.s.
\end{assumption}

\begin{assumption}\label{asmpt:noise-heavy}
    The gradient estimator is unbiased and the noise has bounded moment of order $p \in (1,2]$, i.e., $\E\big[g_t\: \vert \: \mathcal{F}_{t}\big] = \nabla f(x_t)$ and $\E\big[\|z_t\|^p\: \vert \: \mathcal{F}_{t}\big] \leq \sigma^p$, a.s., for all $t \geq 1$ and some $\sigma > 0$.
\end{assumption}

\paragraph{}Assumption \ref{asmpt:noise-bounded} is often used in the context of adaptive methods and relaxed smoothness, e.g., \cite{harvey2019tight,zhang2020gradient,zhang2020improved,li2023adam,pmlr-v247-carmon24a}. For example, in batch learning, where $f(x) = \frac{1}{m}\sum_{i \in [m]}\ell(x;\xi^i)$, the \sfo estimator automatically satisfies Assumption \ref{asmpt:noise-bounded} if $\ell$ has bounded gradients (see Appendix \ref{app:bdd-noise} for details). Assumption \ref{asmpt:noise-heavy} is the standard heavy-tailed noise condition, e.g., \citet{nguyen2023improved,sadiev2023highprobability,hubler2025normalization}. Moreover, many works empirically show that noise satisfying Assumption \ref{asmpt:noise-heavy} frequently arises during training of neural networks and transformers, across a wide range of model architectures and datasets, see, e.g., \cite{simsekli2019tail,zhang2020adaptive} and references therein.   

\subsection{Large Deviations Principle Upper Bound for \sgd}\label{subsec:main-sgd}

In this subsection, we establish an LDP upper bound on $F_t = \min_{k \in [t]}\|\nabla f(x_k)\|^2$. Before stating the main theorem, we define an important concept and provide a useful technical result.

\begin{definition}\label{def:sub-gauss}
    A random vector $z \in \R^d$ is said to be $\sigma$-sub-Gaussian if $\E\Big[\exp\Big(\frac{\|z\|^2}{\sigma^2} \Big)\Big] \leq \exp(1)$.
\end{definition}

\paragraph{}Sub-Gaussian (i.e., light-tailed) noise is widely used for deriving HP convergence, e.g., \cite{ghadimi2013stochastic,li2020high,liu2023high}. We then have the following result.

\begin{lemma}\label{lm:sub-gauss}
    Let Assumption \ref{asmpt:noise-bounded} hold. Then the following are true, for any $t \geq 1$. 
    \begin{enumerate}
        \item The noise is $M$-sub-Gaussian, i.e., we have $\E\Big[\exp\Big(\frac{\|z_t\|^2}{M^2} \Big) \: \big\vert \: \mathcal{F}_t \Big] \leq \exp(1)$.

        \item For any $\mathcal{F}_t$-measurable vector $x \in \R^d$, we have $\E\big[\exp\big(\langle x,z_t\rangle \big) \: \vert \: \mathcal{F}_t \big] \leq \exp\Big(\frac{3M^2\|x\|^2}{4} \Big)$.
    \end{enumerate}
\end{lemma}

\paragraph{}Lemma \ref{lm:sub-gauss} provides some useful bounds on the MGF of the noise. We next state the main result.

\begin{theorem}\label{thm:main-non-conv}
    Let Assumptions \ref{asmpt:init}-\ref{asmpt:noise-bounded} hold and let $\{x_t\}_{t \in \N}$ be the sequence generated by \sgd, using the step-size $\alpha_t = \frac{a}{\sqrt{t+1}}$, where $a \leq \frac{1}{L}$. Then the sequence $\{F_t\}_{t \in \N}$ satisfies an LDP upper bound, with decay rate $n_t = \frac{t}{\log(t)}$ and rate function $I_v: \R \mapsto [0,\infty]$, i.e., for any measurable set $B \subseteq \R$
    \begin{equation*}
        \limsup_{t \rightarrow \infty}\frac{\log(t)}{t}\log\Prob(F_t \in B) \leq -\inf_{x \in \overline{B}}I_v(x),
    \end{equation*} where the rate function is given by $I_v(x) = \begin{cases} 
    \frac{x^2}{24M^2G^2}, & x \geq 0 \\
    +\infty, & x < 0
    \end{cases}$.
\end{theorem}

\paragraph{}For the special case $B = (\epsilon,\infty)$, where $\epsilon > 0$, we immediately have the following corollary.

\begin{corollary}\label{cor:sgd}
    Let conditions of Theorem \ref{thm:main-non-conv} hold. We then have, for any $\epsilon > 0$
    \begin{equation*}
        \limsup_{t \rightarrow \infty}\frac{\log(t)}{t}\log \Prob(F_t > \epsilon) \leq -\frac{\epsilon^2}{24M^2G^2}.
    \end{equation*}
\end{corollary}

Theorem \ref{thm:main-non-conv} and Corollary \ref{cor:sgd} establish the long-term tail decay for \sgd under bounded noise, at rate $e^{-\frac{t}{\log(t)}}$. We next discuss several different aspects of our results.

\paragraph{On the rate function.} Note that the rate function $I_v$ has two distinct regimes: for $x \geq 0$, it is of the form $I_v(x) = \frac{x^2}{24M^2G^2}$, while for any $x < 0$, it takes the value $I_v(x) = +\infty$. In practice, this means that for any set $B \subseteq \R$ whose closure contains a non-negative value, there is a probability that $F_t$ visits the set, but it decays exponentially at rate $t/\log(t)$. On the other hand, for any set $B \subseteq (-\infty,-a)$, for $a > 0$, we can see that $\inf_{x \in \overline{B}} I(x) = +\infty$, hence $\limsup_{t \rightarrow \infty}\frac{\log(t)}{t}\log\Prob(F_t \in B) = -\infty$, implying that $\Prob(F_t \in B) = 0$ in the long run, which is expected, as $F_t \geq 0$. 

\paragraph{Dependence on problem parameters.} We can see that the rate function depends on two problem parameters, the noise and gradient bounds $M$ and $G$, decreasing as either grows. This is again expected, as the convergence of \sgd slows down with more noise (measured by $M$) and a more complex cost (measured by $G$), resulting in a smaller leading constant in the tail decay rate, meaning that it takes more time for $F_t$ to escape a set $B$. Finally, while the leading term in finite-time bounds depends on parameters like the initial optimality gap $f(x_1) - f_\star$ and smoothness $L$ (see Section \ref{sec:discus} for details), our results indicate that these parameters do not affect the asymptotic decay.   

\paragraph{On the metric.} While our results are presented in terms of $F_t =\min_{k \in [t]}\|\nabla f(x_k)\|^2$, they continue to hold for the average norm-squared, i.e., $A_t = \frac{1}{t}\sum_{k = 1}^t\|\nabla f(x_k)\|^2$, which is more general, as $F_t \leq A_t$. In our proofs, the results are first established in terms of $A_t$, which then implies the same bounds on $F_t$, while the choice of presenting the results in terms of $F_t$ stems from it being standard in non-convex optimization and its ease of interpretability, see Appendix \ref{app:metric} for details. 

\paragraph{On the noise.} Compared to existing HP results for vanilla \sgd, e.g., \cite{liu2023high}, which consider sub-Gaussian noise, we impose a slightly stronger condition of a.s. bounded noise. As discussed in the introduction, this is a byproduct of the fact that we need to bound the MGF $\E[\exp(\lambda F_t)]$ over the entire domain, i.e., for every $\lambda \in \R$, while HP results only require bounding the MGF locally, e.g., over a compact domain, or at $\lambda = 1$ (see the discussion in Section \ref{sec:discus} for details). As we show in the next subsection, this condition can be significantly relaxed, by applying a (bounded) nonlinear operator to \sgd, ensuring that the effective noise remains bounded.  

\subsection{Large Deviations Principle Upper Bound Under Heavy-Tailed Noise}\label{subsec:main-clip}

In this subsection, we relax the noise assumption by considering the \csgd method and provide an LDP upper bound under heavy-tailed noise. For ease of notation, let $\widetilde{g}_t \coloneqq \min\Big\{1,\frac{\gamma_t}{\|g_t\|}\Big\}g_t$ denote the clipped stochastic gradient and let $\theta_t^u \coloneqq \widetilde{g}_t - \E\big[\widetilde{g}_t \: \vert \: \mathcal{F}_t \big]$ and $\theta_t^b \coloneqq \E\big[\widetilde{g}_t \: \vert \: \mathcal{F}_t \big] - \nabla f(x_t)$ respectively be the unbiased and biased components of the difference of the clipped stochastic gradient and the true gradient (i.e., \emph{clipping bias}), noting that $\widetilde{g}_t - \nabla f(x_t) = \theta_t^u + \theta_t^b$. Decomposing the clipping bias into an unbiased and biased component is standard when analyzing \csgd, see, e.g., \cite{sadiev2023highprobability,nguyen2023improved}. We then have the following important result.

\begin{lemma}\label{lm:clip-bias-bound}
    Let Assumptions \ref{asmpt:cost} and \ref{asmpt:noise-heavy} hold and let the clipping threshold be chosen as
    \begin{equation}\label{eq:clipping-radius}
        \gamma_t = \begin{cases}
            2G(t+1)^{\frac{2-p}{6p-4}}, & p \in (1,2) \\
            2G\sqrt{\log(t+1)}, & p = 2.
        \end{cases}
    \end{equation} Then the following are true, for any $t \geq 1$.
    \begin{enumerate}
        \item $\|\theta_t^b\| \leq 4\sigma^p\gamma_t^{1-p}$.

        \item For any $\mathcal{F}_t$-measurable $x \in \R^d$, we have $\E\big[\exp\big(\langle x, \theta_t^u\rangle \big) \: \vert \: \mathcal{F}_t \big] \leq \exp\big(3\gamma_t^2\|x\|^2\big)$.
    \end{enumerate}
\end{lemma}

\paragraph{}Lemma \ref{lm:clip-bias-bound} provides a bound on the biased component and establishes sub-Gaussian concentration of the unbiased one, facilitating the rest of our analysis. We note that $\gamma_t$ can be tuned without requiring knowledge of $G$, with the choice in \eqref{eq:clipping-radius} simplifying the exposition (see the discussion after Theorem \ref{thm:main-non-conv-clip} and Appendix \ref{app:clip-relax} for more details). We next state the main result. 

\begin{theorem}\label{thm:main-non-conv-clip}
    Let Assumptions \ref{asmpt:init}, \ref{asmpt:cost} and \ref{asmpt:noise-heavy} hold and let $\{x_t\}_{t \in \N}$ be the sequence generated by \csgd using the step-size $\alpha_t = (t+1)^{-\frac{p}{3p-2}}$ and clipping threshold $\gamma_t$ given in \eqref{eq:clipping-radius}. Then the sequence $\{F_t\}_{t \in \N}$ satisfies an LDP upper bound, with the following decay rate and rate function.
    \begin{enumerate}
        \item If $p \in (1,2)$, the decay rate is $n_t = \frac{t^{\frac{4(p-1)}{3p-2}}}{\log(t)}$, with rate function given by $I_c(x) = \begin{cases}
            \frac{x^2}{768G^4}, & x \geq 0 \\
            +\infty, & x < 0.
        \end{cases}$

        \item If $p = 2$, the decay rate is $n_t = \frac{t}{\log^2(t)}$, with rate function given by $I_c(x) = \begin{cases}
            \frac{x^2}{384G^4}, & x \geq 0 \\
            +\infty, & x < 0.
        \end{cases}$
    \end{enumerate}
\end{theorem}

\paragraph{}Similarly to the previous section, if $B = (\epsilon,\infty)$, where $\epsilon > 0$, we have the following result.

\begin{corollary}\label{cor:clip}
    Let conditions of Theorem \ref{thm:main-non-conv-clip} hold. Then the following are true, for any $\epsilon > 0$.
    \begin{enumerate}
        \item If $p \in (1,2)$, then $\limsup_{t \rightarrow \infty}\frac{\log(t)}{t^{\beta_p}}\log \Prob(F_t > \epsilon) \leq -\frac{\epsilon^2}{768G^4}$, where $\beta_p = \frac{4(p-1)}{3p-2}$.

        \item If $p = 2$, then $\limsup_{t \rightarrow \infty}\frac{\log^2(t)}{t}\log \Prob(F_t > \epsilon) \leq -\frac{\epsilon^2}{384G^4}.$
    \end{enumerate}
\end{corollary}

\paragraph{}Theorem \ref{thm:main-non-conv-clip} and Corollary \ref{cor:clip} establish the long-term tail decay for \csgd under heavy-tailed noise, notably showing the rate $e^{-t/\log^2(t)}$ for noise with bounded variance. We next discuss the results.

\paragraph{On the decay rate.} We can see that the decay rate in Theorem \ref{thm:main-non-conv-clip} has two distinct regimes: for $p \in (1,2)$, the decay rate is of order $e^{-t^{\beta_p}/\log(t)}$, where $\beta_p = \frac{4(p-1)}{3p-2}$, while for $p = 2$, we incur an additional $\log(t)$ factor, showing the decay rate $e^{-t/\log^2(t)}$. This is consistent with MSE and HP convergence results established under Assumption \ref{asmpt:noise-heavy}, in the sense that the convergence rate exponent explicitly depends on the noise moment $p$, see, e.g., \cite{zhang2020adaptive,nguyen2023improved}. 

\paragraph{On the rate function and dependence on problem parameters.} Similarly to the discussion in the previous subsection, the rate function $I_c$ has two distinct regimes for negative and non-negative values. On the other hand, while it exhibits dependence on problem parameters, it does so only through the gradient bound $G$. We note that this stems from our choice of clipping threshold in \eqref{eq:clipping-radius} and the fact that the unbiased component of the clipping bias is $\gamma_t$-sub-Gaussian (recall Lemma \ref{lm:clip-bias-bound}).

\paragraph{On the metric.} Similarly to the previous subsection, our results can be equally stated in terms of the average norm-squared of the gradients, $\frac{1}{t}\sum_{k = 1}^t\|\nabla f(x_k)\|^2$, see Appendix \ref{app:metric} for details.  
 
\paragraph{On the clipping threshold.} We use an increasing clipping threshold in \eqref{eq:clipping-radius}, which is consistent with existing works \cite{cutkosky2021high,nguyen2023improved}. However, compared to \cite{cutkosky2021high,nguyen2023improved}, who use the clipping threshold $\gamma_t = \widetilde{\mathcal{O}}\Big(t^{\frac{1}{3p-2}}\Big)$,\footnote{Here, we use $\widetilde{\mathcal O}(\cdot)$ to hide problem related constants and terms poly-logarithmic in $t$.} our clipping threshold increases at a strictly slower rate (e.g., for $p = 2$ our threshold increases at rate $\sqrt{\log(t)}$, while in the said works it increases at rate $t^{1/4}$), making the likelihood of clipping higher, further closing the gap on how clipping is used in practice.\footnote{Contrary to the increasing threshold used in theory, in practice clipping is used with a small, constant threshold, see, e.g., \cite{zhang2022opt,touvron2023llama,liu2024deepseek}.} Next, we assume knowledge of noise moment $p$ and gradient bound $G$ to tune the clipping threshold in \eqref{eq:clipping-radius}, which is on par with \cite{cutkosky2021high}, while \cite{nguyen2023improved} relax the bounded gradient condition, at the expense of require knowledge of noise moment $p$ and parameter $\sigma$, the initial optimality gap $f(x_1) - f_\star$ and smoothness constant $L$. Finally, we note that knowledge of $G$ is not necessary in \eqref{eq:clipping-radius} for our results to hold. In particular, our results continue to hold if the clipping threshold is selected as 
\begin{equation*}
    \gamma_t = \begin{cases}
            C(t+1)^{\frac{2-p}{6p-4}}, & p \in (1,2) \\
            C\sqrt{\log(t+1)}, & p = 2,
        \end{cases}
\end{equation*} where $C > 0$ is any value. Using this threshold, it can be shown that the resulting rate function will be of the form $I_c(x) = \mathcal{O}\Big(\frac{x^2}{C^2G^2}\Big)$ for $x \geq 0$, and $I_c(x) = +\infty$ otherwise, where $\mathcal{O}(\cdot)$ hides global constants. The reader is referred to Appendix \ref{app:clip-relax} for a formal statement and derivations.

\subsection{A Lower Bound on the Tail Probability}\label{subsec:sgd-lbd}

In this subsection, we show that the results in Theorems \ref{thm:main-non-conv} and \ref{thm:main-non-conv-clip} are tight, by establishing a lower bound on the tail probability induced by \sgd/\csgd. To that end, we have the following result.

\begin{theorem}\label{thm:low-bdd}
    There exist a cost, initialization and \sfo obeying Assumptions \ref{asmpt:init}-\ref{asmpt:noise-bounded}, as well as a problem dependent constant $b > 0$ and global constants $a_1,a_2 > 0$, such that, for any $\epsilon \in (0,b)$ and all $t \geq 1$, the tail probability induced by either \sgd or \csgd, satisfies $\Prob(F_t > \epsilon) \geq a_1e^{-a_2 t}$.
\end{theorem}

\paragraph{}It readily follows that Theorem \ref{thm:low-bdd} implies the following asymptotic lower bound, for any $\epsilon \in (0,b)$
\begin{equation*}
    \liminf_{t \rightarrow \infty}\frac{1}{t}\log\Prob(F_t > \epsilon) \geq -a_2.
\end{equation*} Theorem \ref{thm:low-bdd} shows that there exist hard problem instances for which the tail probability of both \sgd and \csgd \emph{can not decay faster than} $e^{-t}$. We next discuss the result from several perspectives.

\paragraph{On the tightness of decay rates.} Combining the results of Theorem \ref{thm:low-bdd} with Corollary \ref{cor:sgd}, it follows that for any $\epsilon \in (0,b)$, the tails induced by \sgd satisfy
\begin{equation}\label{eq:near-ldp}
    -a_2 \leq \liminf_{t \rightarrow \infty} \frac{1}{t}\log\Prob(F_t > \epsilon) \: \text{ and } \: \limsup_{t \rightarrow \infty} \frac{\log(t)}{t}\log\Prob(F_t > \epsilon) \leq -\frac{\epsilon^2}{24M^2G^2}.
\end{equation} Equation \eqref{eq:near-ldp} demonstrates that the long-term tail decay rate of \sgd in Theorem \ref{thm:main-non-conv} is tight up to a $\log(t)$ factor. Similarly, combined with Corollary \ref{cor:clip} for $p = 2$, it can be seen that for any $\epsilon \in (0,b)$
\begin{equation}\label{eq:near-ldp-clip}
    -a_2 \leq \liminf_{t \rightarrow \infty} \frac{1}{t}\log\Prob(F_t > \epsilon) \: \text{ and } \: \limsup_{t \rightarrow \infty} \frac{\log^2(t)}{t}\log\Prob(F_t > \epsilon) \leq -\frac{\epsilon^2}{384G^4},
\end{equation} underlining the tightness of the asymptotic tail decay rate of \csgd for $p = 2$, up to a $\log^2(t)$ factor. 

\paragraph{On the rate function optimality.} While \eqref{eq:near-ldp} and \eqref{eq:near-ldp-clip} indicate the the tightness of the decay rates, they are weaker than a full LDP of the form in \eqref{eq:ldp}, in the sense that they do not establish matching upper and lower bounds on the asymptotic probability of failure. As such, it is unclear if the rate functions in Theorems \ref{thm:main-non-conv} and \ref{thm:main-non-conv-clip} depend optimally on the problem parameters. However, as we discuss in the next section, our rate functions exhibit dependence on similar problem parameters as the corresponding decay constants obtained from finite-time high-probability bounds.

\section{Comparison with State-Of-The-Art}\label{sec:discus}

In this section, we provide detailed comparison with state-of-the-art (SOTA) long-term tail decay results for \sgd-type methods in non-convex optimization stemming from either HP or LDP guarantees, as well as a further discussion on the differences in the assumptions.

\paragraph{SOTA results for \sgd.}\cite{liu2023high} provide SOTA finite-time HP convergence guarantees of \sgd, for a $L$-smooth, bounded from below cost, and $B$-sub-Gaussian noise. Using the time-varying step-size $\alpha_t = \frac{1}{L\sqrt{t}}$, they show the following result, for any $\delta \in (0,1)$ and any $t \geq 1$
\begin{equation*}
    \Prob\bigg(F_t > \frac{2\Delta L + 3B^2(1+\log(t)) + 12B^2\log(\nicefrac{1}{\delta})}{\sqrt{t}}\bigg) \leq \delta,
\end{equation*} where $\Delta \coloneqq f(x_1) - f_\star$ is the optimality gap of the initial model. The above bound leads to the following tail result $\Prob\big(F_t > \epsilon\big) \leq \exp\Big(-\frac{\epsilon\sqrt{t}}{12B^2} + \frac{\Delta L}{6B^2} + \frac{1 + \log(t)}{4}\Big)$, for any $\epsilon > 0$ and any $t \geq 1$. Taking the logarithm, dividing everything by $\sqrt{t}$ and taking the $\limsup$, we get  
\begin{equation}\label{eq:liu-ldp}
    \limsup_{t \rightarrow \infty}\frac{1}{\sqrt{t}}\log\Prob\big(F_t > \epsilon\big) \leq -\frac{\epsilon}{12B^2}.
\end{equation} Comparing \eqref{eq:liu-ldp} to the bound in Corollary \ref{cor:sgd}, we can see that both results depend on the noise (via $B$ and $M$), with our result further depending on $G$. More importantly, the long-term decay rate in \eqref{eq:liu-ldp} is of the order $\sqrt{t}$, while the one in Corollary \ref{cor:sgd} is of the order $\frac{t}{\log(t)}$, an order of magnitude faster. This improvement stems from the different approach taken in our work, focusing directly on long-term tail bounds, while on the other hand, \cite{liu2023high} focus on finite-time HP results, that hold for any $t$, but result in loose long-term bounds. Finally, we note that while our rate function depends on the gradient bound $G$, which can be large in some applications, this becomes negligible relative to the gain in the decay rate, especially in the long-term regime $t \rightarrow \infty$ considered in our work.

\paragraph{SOTA results for \csgd.} \cite{nguyen2023improved} provide SOTA finite-time HP convergence guarantees of \csgd, for a $L$-smooth, bounded from below cost, and noise with bounded $p$-th moment. Using the time varying step-size $\alpha_t = \widetilde{\mathcal{O}}\Big(t^{-\frac{p}{3p-2}}\Big)$ and clipping threshold $\gamma_t = \widetilde{\mathcal{O}}\Big(t^{\frac{1}{3p-2}}\Big)$,\footnote{Here $\widetilde{\mathcal{O}}(\cdot)$ hides global and problem related constants, as well as terms poly-logarithmic in $t$ and $\nicefrac{1}{\delta}$.} they show the following result, for any $\delta \in (0,1/e)$ and $t \geq 1$
\begin{equation*}
    \Prob\bigg(F_t > \frac{720\sigma\sqrt{\Delta L}\log^{\frac{2p}{3p-2}}(t)\log(\nicefrac{1}{\delta})}{t^{\beta_p/2}}\bigg) \leq \delta,
\end{equation*} where we recall that $\beta_p = \frac{4(p-1)}{3p-2}$.\footnote{The bound from \cite{nguyen2023improved} is simplified, for ease of presentation. It can be shown that the tail decay rate stemming from their bounds actually deteriorates to a slower one as $p \rightarrow 1$, which we omit and present the faster rate.} The above bound leads to the following tail result $\Prob\big(F_t > \epsilon\big) \leq \exp\Big(-\frac{\epsilon t^{\beta_p/2}}{720\sigma\sqrt{\Delta L}\log^{2p/(3p-2)}(t)}\Big)$, for any $\epsilon > 0$ and any $t \geq 1$. Taking the logarithm, dividing everything by $t^{\beta_p/2}/\log^{\frac{2p}{3p-2}}(t)$ and taking the $\limsup$, we get  
\begin{equation}\label{eq:nguyen-ldp}
    \limsup_{t \rightarrow \infty}\frac{\log^{\frac{2p}{3p-2}}}{t^{\beta_p/2}}\log\Prob\big(F_t > \epsilon\big) \leq -\frac{\epsilon}{720\sigma\sqrt{\Delta L}}.
\end{equation} Comparing \eqref{eq:nguyen-ldp} to the bound in Corollary \ref{cor:clip}, we can see that our result depends on $G$, while the result in \eqref{eq:nguyen-ldp} depends on the noise, initial optimality gap and smoothness (with the latter two stemming from the choice of the clipping threshold in \cite{nguyen2023improved}). Ignoring the $\log$ factors, we can again see that the long-term decay rate in \eqref{eq:nguyen-ldp} is of order $t^{\beta_p/2}$, with the one in Corollary \ref{cor:clip} being of order $t^{\beta_p}$, an order of magnitude faster. Similarly, \cite{nguyen2023improved} provide bounds that hold for any $t$ and do not require bounded gradients, however, the resulting long-term tail bounds are very loose. Importantly, the finite-time HP bound established in \cite{nguyen2023improved} matches that of \csgd from \cite{cutkosky2021high}, who additionally require bounded gradients. As such, the assumptions used in \cite{cutkosky2021high} are the same as in our work, while their tail decay rate is the same as in \cite{nguyen2023improved} and an order of magnitude slower than ours. This further highlights that the improved decay rates shown in our work are not simply a result of stronger assumptions, but of a \emph{fundamentally different approach in studying the long-term tail decay}. 

Finally, \cite{armacki2024_ldp+mse} provide an LDP upper bound for a family of nonlinear \sgd methods, which, among others, includes clipping. Under $L$-smooth, lower bounded cost, noise with a symmetric PDF that is strictly positive around the origin, using the step-size $\alpha_t = \frac{1}{\sqrt{t+1}}$ and a constant threshold $\gamma_t = C$, the authors show that
\begin{equation}
    \limsup_{t \rightarrow \infty}\frac{\log(t)}{\sqrt{t}}\log\Prob(F_t > \epsilon) \leq -\frac{\min\{\epsilon,\sqrt{\epsilon}\}}{16C^4L^2},
\end{equation} implying a long-term tail decay rate $\frac{\sqrt{t}}{\log(t)}$, which is strictly worse than the rate in Corollary \ref{cor:clip} for any $p > 6/5$. Crucially, the LDP upper bound in \cite{armacki2024_ldp+mse} is derived for a black-box nonlinear framework, under very different noise conditions, namely IID noise with symmetric PDF. 

\paragraph{On the assumptions.} As was already mentioned, we require uniformly bounded noise for vanilla \sgd, which is stronger than the sub-Gaussian condition used in \cite{liu2023high}. This stronger requirement stems from the fact that, in order to invoke the G\"{a}rtner-Ellis theorem, we need to show that the \emph{scaled MGF is bounded everywhere}, i.e., that $\limsup_{t \rightarrow \infty}\frac{1}{n_t}\log\E[\exp(n_t\lambda F_t)] < \infty$, for all $\lambda \in \R$ and $n_t = \frac{t}{\log(t)}$. This is much stronger than the bound for HP results, where it typically suffices to show $\E[\exp(F_t)] \leq \exp\big(\mathcal{O}\big(\sum_{k = 1}^t\alpha_k^2\big)\big)$, i.e., for $n_t = \lambda = 1$, making it much easier to control the MGF. This stronger requirement necessitates uniformly bounded noise for vanilla \sgd, for the following reason. Starting from \eqref{eq:init-step} in the Appendix, pushing $M^2$ inside the sum and replacing it by $\|z_k\|^2$, we have $F_t = A_t + C\sum_{k = 1}^t\alpha_k^2\|z_k\|^2$. Ignoring the term $A_t$ and assuming that the noise is sub-Gaussian, we then get $\E[\exp(n_t\lambda F_t)] \leq \E[\exp(n_t\lambda\sum_{k = 1}^t\alpha_k^2\|z_k\|^2)] = \E[\exp(n_t\lambda M^2\sum_{k = 1}^t\alpha_k^2\|z_k\|^2/M^2)]$. To control this term, we need $h_{t,\lambda} \coloneqq n_t\alpha_k^2\lambda M^2 \leq 1$ to hold for all $\lambda \in \R$ and $t\geq k \geq 1$. Recalling that $n_t = \frac{t}{\log(t)}$ and $\lambda$ can be arbitrarily large, one can see that $h_{t,\lambda}$ is unbounded, necessitating uniformly bounded noise. This is not required in the analysis of \csgd, as the resulting operator is bounded, facilitating heavy-tailed noise. Similarly, the gradient bound is required to control the terms $\langle \nabla f(x_t), z_t\rangle$ and $\langle \nabla f(x_t), \theta^u_t\rangle$ for \sgd and \csgd, respectively. While we believe that it might be possible to fully remove the bounded gradient condition for \csgd, it is beyond the scope of the current work and represents an important future direction.

\section{Conclusion}\label{sec:conc}

We studied the long-term tail decay and probability of failure of \sgd-type methods in non-convex optimization, demonstrating that, in the long run, the tails decay at an order of magnitude faster rate than suggested by existing works. Further, we show that our results are tight, by establishing lower bounds on the tail probability, which match our upper bounds up to poly-logarithmic factors. As such, our results provide much sharper bounds on the probability of failure in non-convex optimization, implying stronger guarantees for individual runs of \sgd-based methods in applications that require an enormous number of iterations, such as training of deep learning models and LLMs. Future work includes removing assumptions like the uniformly bounded gradients for \csgd, as well as establishing a full LDP with matching lower tail decay rate and rate function.

\bibliography{bibliography}

\appendix

\section{Introduction}

The appendix contains results omitted from the main body. Appendix \ref{app:literature} reviews the literature on other useful performance guarantees, Appendix \ref{app:bdd-noise} discusses when noise satisfying Assumption \ref{asmpt:noise-bounded} arises naturally, Appendix \ref{app:prelim} provides some important intermediary results, Appendix \ref{app:non-conv} contains the proof of Theorem \ref{thm:main-non-conv}, Appendix \ref{app:clip} provides the proof of Theorem \ref{thm:main-non-conv-clip}, Appendix \ref{app:low-bdd} contains the proof of Theorem \ref{thm:low-bdd}, Appendix \ref{app:metric} discusses the choice of metric, while Appendix \ref{app:clip-relax} gives results for \csgd with a clipping threshold that does not requiring knowledge of the gradient bound.

\section{Other Performance Guarantees}\label{app:literature}

In this section we provide a brief overview of other useful performance guarantees encountered in the literature on \sgd-type methods. Perhaps the most frequent among them are \emph{MSE guarantees}, which characterize the average behaviour across many runs, with classical results establishing convergence under the bounded variance assumption (i.e., $p = 2$ in Assumption \ref{asmpt:noise-heavy}), e.g., \cite{rakhlin2012making,ghadimi2012optimal,ghadimi2013stochastic,liu2020improved,liu2024revisiting}, with works like \cite{chen-decent-pareto,khaled2022better} allowing the second noise moment to grow with the gradient norm and/or optimality gap, while \cite{zhang2020adaptive,jakovetic2023nonlinear,hubler2025normalization,liu2025nonconvex,sun2025heavy} study MSE guarantees under heavy-tailed noise. Another popular guarantee is \emph{almost sure convergence}, where the goal is to show convergence with probability one, e.g., \cite{bertsekas-gradient,li2019almost-sure,cevher-almost_sure,pmlr-v134-sebbouh21a,jakovetic2023nonlinear,armacki2024_ldp+mse}. Finally, some recent works focus on \emph{low-probability convergence}, i.e., guarantees of the form $\Prob\big(F_t > \frac{1}{\delta n_t}\big) \leq \delta$, encountered in the context of convergence of adaptive methods or generalized smoothness, see, e.g., \cite{li2023gen-smooth,li2023adam}. While important in their own rights, none of the said guarantees provide tight bounds on the tail probability, with decay at exponential scale, as is the case with HP and LD-style guarantees.

\section{When Assumption \ref{asmpt:noise-bounded} Arises Naturally}\label{app:bdd-noise}

In this section, we elaborate on when the noise induced by the \sfo satisfies Assumption \ref{asmpt:noise-bounded}. As mentioned in Section \ref{subsec:assump}, we consider the batch setting, where the cost $f$ is of the form $f(x) = \frac{1}{m}\sum_{i \in [m]}\ell(x;\xi^i)$, for some finite dataset $\{\xi^i\}_{i \in [m]}$ and the loss $\ell$ has $G$-bounded gradients, i.e., $\|\nabla \ell(x;\xi^i)\| \leq G$, for all $x \in \R^d$ and every $i \in [m]$ (e.g., satisfied by any $G$-Lipschitz loss, which, as discussed in Section \ref{subsec:assump}, contains for a broad class of Lipschitz continuous deep neural networks and transformers \cite{fazlyab2019lipschitz,dongmian2020lipschitz,combettes2020lipschitz,zhang2022lipschitz,pmlr-v139-kim21i}). It follows that the gradient of $f$ is given by
\begin{equation}\label{eq:batch}
    \nabla f(x) = \frac{1}{m}\sum_{i \in [m]}\nabla \ell(x;\xi^i).
\end{equation} As discussed in Section \ref{subsec:oracle}, when queried in iteration $t$ with input $x_t$, the \sfo returns the estimator 
\begin{equation}\label{eq:estimator}
    g_t = \frac{1}{|S_t|}\sum_{j \in S_t}\nabla \ell(x_t;\xi^j),    
\end{equation} where $S_t \subset [m]$ is a set of indices drawn uniformly at random. We now want to verify that $g_t$ satisfies Assumption \ref{asmpt:noise-bounded}. By the definition of the \sfo, we have
\begin{equation}\label{eq:unbiased}
    \E[g_t \: \vert \: \mathcal{F}_t] = \frac{1}{|S_t|}\sum_{i \in [m]}\nabla \ell(x_t;\xi^i)\Prob(i \in S_t\:\vert\:\mathcal{F}_t) = \nabla f(x_t),
\end{equation} where the last equality follows from the fact that $\Prob(i \in S_t\:\vert\: \mathcal{F}_t) = \frac{|S_t|}{m}$ and \eqref{eq:batch}. Moreover, using \eqref{eq:batch}-\eqref{eq:estimator} and the triangle inequality, it can be seen that the noise $z_t = g_t - \nabla f(x_t)$ satisfies
\begin{equation}\label{eq:bounded}
    \|z_t\| \leq \|g_t\| + \|\nabla f(x_t)\| \leq \frac{1}{|S_t|}\sum_{j \in S_t}\|\nabla \ell(x_t;\xi^j)\| + \frac{1}{m}\sum_{i \in [m]}\|\nabla \ell(x_t;\xi^i)\| \leq 2G,
\end{equation} where the last inequality follows from the fact that the loss $\ell$ has $G$-bounded gradients. Equations \eqref{eq:unbiased} and \eqref{eq:bounded} readily imply that Assumption \ref{asmpt:noise-bounded} holds for the noise induced by the \sfo, as claimed.

\section{Intermediate Results}\label{app:prelim}

We start by stating an important result, crucial to establishing LDP-style bounds, known as the G\"{a}rnter-Ellis theorem, see, e.g., \citet{dembo2009large}. 

\begin{proposition}\label{prop:gartner-ellis}
    Let $\Lambda_t: \R \mapsto \R$ be a sequence of log moment-generating functions induced by a sequence of measures $\mu_t: \mathcal{B}(\R) \mapsto [0,1]$, $t \in \N$. If for some $\{n_t\}_{t \in \N}$, such that $n_t > 0$ and $\lim_{t \rightarrow \infty}n_t = \infty$, and each $\lambda \in \R$, we have
    \begin{equation*}
        \limsup_{t \rightarrow \infty}\frac{\Lambda_t(n_t\lambda)}{n_t} \leq \varphi(\lambda) < \infty,
    \end{equation*} then $\{\mu_t\}_{t \in \N}$ satisfies the LDP upper bound with decay rate $n_t$ and rate function $I: \R \mapsto [0,\infty]$, given by the Fenchel-Legendre transform of $\varphi$, i.e., $I(x) = \varphi^\star(x) = \sup_{\lambda \in \R}\left\{ x\lambda - \varphi(\lambda) \right\}$.
\end{proposition}

Next, we prove Lemma \ref{lm:sub-gauss}. For completeness, we restate it below.

{
\renewcommand{\thetheorem}{2}
\begin{lemma}
    Let Assumption \ref{asmpt:noise-bounded} hold. Then the following are true, for any $t \geq 1$. 
    \begin{enumerate}
        \item The noise is $M$-sub-Gaussian, i.e., we have $\E\Big[\exp\Big(\frac{\|z_t\|^2}{M^2} \Big) \: \big\vert \: \mathcal{F}_t \Big] \leq \exp(1)$.

        \item For any $\mathcal{F}_t$-measurable vector $x \in \R^d$, we have $\E\big[\exp\big(\langle x,z_t\rangle \big) \: \vert \: \mathcal{F}_t \big] \leq \exp\Big(\frac{3M^2\|x\|^2}{4} \Big)$.
    \end{enumerate}
\end{lemma}
}

\begin{proof}
    The first claim follows directly from Assumption \ref{asmpt:noise-bounded} and Definition \ref{def:sub-gauss}. To prove the second claim, we follow a similar approach to, e.g., \citet[Lemma 1]{li2020high}. For ease of notation, let $y_t \coloneqq \frac{z_t}{M}$ and note that from the first claim we have 
    \begin{equation}\label{eq:sub-gauss-y}
        \E\big[\exp(\|y_t\|^2) \: \vert \: \mathcal{F}_t\big] \leq \exp(1).
    \end{equation} Next, let $x \in \R^d$ be $\mathcal{F}_t$-measurable and assume first that $\|x\| \leq \frac{4}{3}$. Using the inequality $\exp(a) \leq a + \exp(\nicefrac{9a^2}{16})$, which holds for any $a \in \R$, we then have
    \begin{align*}
        \E\lbr\exp\lp\langle x,y_t \rangle\rp \: \vert \: \mathcal{F}_t\rbr &\leq \E\lbr \langle x,y_t \rangle + \exp\lp\frac{9\langle x,y_t \rangle^2}{16}\rp \: \big\vert \: \mathcal{F}_t\rbr \stackrel{(a)}{=} \E\lbr\exp\lp\frac{9\langle x,y_t \rangle^2}{16}\rp \: \big\vert \: \mathcal{F}_t\rbr \\
        &\stackrel{(b)}{\leq} \E\lbr\exp\lp\frac{9\|x\|^2\|y_t\|^2 }{16}\rp \: \big\vert \: \mathcal{F}_t\rbr \stackrel{(c)}{\leq} \lp\E\lbr\exp\lp\|y_t\|^2\rp \: \vert \: \mathcal{F}_t\rbr\rp^{\nicefrac{9\|x\|^2}{16}} \\
        &\stackrel{(d)}{\leq} \exp\lp \frac{9\|x\|^2}{16} \rp \leq \exp\lp \frac{3\|x\|^2}{4} \rp,
    \end{align*} where $(a)$ follows from the facts that $x$ is $\mathcal{F}_t$-measurable and the noise is unbiased, $(b)$ follows from the Cauchy-Schwartz inequality, in $(c)$ we use the fact that $\frac{9\|x\|^2}{16} \leq 1$ and Jensen's inequality, while $(d)$ follows from \eqref{eq:sub-gauss-y}. On the other hand, if $\|x\| > \frac{4}{3}$, we use Young's inequality, i.e., $ab \leq \frac{a^2}{2\epsilon} + \frac{\epsilon b^2}{2}$, with $\epsilon = \frac{4}{3}$, to get 
    \begin{align*}
        \E[\exp(\langle x,y_t\rangle) \: \vert \: \mathcal{F}_t] &\leq \exp\lp \frac{3\|x\|^2}{8} \rp\E\lbr\exp\lp \frac{2\|y_t\|^2}{3} \rp \: \big\vert \: \mathcal{F}_t \rbr \\
        &\leq \exp\lp \frac{2}{3} + \frac{3\|x\|^2}{8} \rp \leq \exp\lp \frac{3\|x\|^2}{4} \rp,
    \end{align*} where the second inequality follows from Jensen's inequality and \eqref{eq:sub-gauss-y}, while the third inequality follows from the fact that $\frac{2}{3} < \frac{3\|x\|^2}{8}$, since $\|x\| > \frac{4}{3}$. Combining both cases, we get 
    \begin{equation}\label{eq:last-step}
        \E[\exp\lp\langle x,y_t \rangle \rp \:\vert\:\mathcal{F}_t] \leq \exp\lp\frac{3\|x\|^2}{4}\rp.    
    \end{equation} The claim follows by noting that $\langle x,z_t \rangle = \langle Mx,y_t\rangle$ and applying \eqref{eq:last-step}.
\end{proof}

Prior to proving Lemma \ref{lm:clip-bias-bound}, we provide a known technical result on the behaviour of the clipping operator, see, e.g., \citet[Lemma 5.1]{sadiev2023highprobability}.

\begin{proposition}\label{prop:sadiev-technical}
    Let $X \in \R^d$ be a random vector and let $\widetilde{X} = \min\Big\{1,\frac{\gamma}{\|X\|}\Big\}X$ be its clipped version. If $\E[X] = x$, $\E\|x - X\|^p \leq \sigma^p$ for some $p \in (1,2]$ and $\|x\| \leq \frac{\gamma}{2}$, then 
    \begin{equation*}
        \big\|\E[\widetilde{X}] - x\big\| \leq 4\sigma^p\gamma^{1-p}.    
    \end{equation*}
\end{proposition}

We are now ready to prove Lemma \ref{lm:clip-bias-bound}, which we restate below, for the reader's convenience.

{
\renewcommand{\thetheorem}{5}
\begin{lemma}
    Let Assumptions \ref{asmpt:cost} and \ref{asmpt:noise-heavy} hold and let the clipping threshold be chosen as
    \begin{equation*}
        \gamma_t = \begin{cases}
            2G(t+1)^{\frac{2-p}{6p-4}}, & p \in (1,2) \\
            2G\sqrt{\log(t+1)}, & p = 2.
        \end{cases}
    \end{equation*} Then the following are true, for any $t \geq 1$.
    \begin{enumerate}
        \item $\|\theta_t^b\| \leq 4\sigma^p\gamma_t^{1-p}$.

        \item For any $\mathcal{F}_t$-measurable $x \in \R^d$, we have $\E\big[\exp\big(\langle x, \theta_t^u\rangle \big) \: \vert \: \mathcal{F}_t \big] \leq \exp\big(3\gamma_t^2\|x\|^2\big)$.
    \end{enumerate}
\end{lemma}
}

\begin{proof}
    To prove the first claim, we use Assumption \ref{asmpt:cost} and the fact that, for any $t \geq 1$ 
    \begin{equation*}
        \|\nabla f(x_t)\| \leq G \leq \frac{\gamma_t}{2},
    \end{equation*} where the second inequality follows from the choice of clipping threshold in \eqref{eq:clipping-radius}. The claim now follows by applying Proposition \ref{prop:sadiev-technical} to $\widetilde{g}_t$. To prove the second claim, note that, by the definition of $\theta_t^u$, we have
    \begin{equation*}
        \|\theta_t^u\| = \|\widetilde{g}_t - \E_t[\widetilde{g}_t]\| \leq \|\widetilde{g}_t\| + \E_t\|\widetilde{g}_t\| \leq 2\gamma_t.
    \end{equation*} The claim now readily follows from Lemma \ref{lm:sub-gauss}, setting $M = 2\gamma_t$.
\end{proof}

\section{Proof of Theorem \ref{thm:main-non-conv}}\label{app:non-conv}

Using the $L$-smoothness and the \sgd update rule~\eqref{eq:sgd} with $\Psi_t(x) = x$, we get
\begin{align*}
    f(x_{k+1}) &\leq f(x_k) + \langle \nabla f(x_k), x_{k+1} - x_k \rangle + \frac{L}{2}\|x_{k+1} - x_k\|^2 \\ 
    &\leq f(x_k) - \alpha_k\left(1 - \frac{\alpha_kL}{2} \right)\|\nabla f(x_k)\|^2 - \alpha_k(1-\alpha_kL)\langle \nabla f(x_k), z_k\rangle + \frac{\alpha_k^2L}{2}\|z_k\|^2 \\ 
    &\leq f(x_k) - \frac{\alpha_k}{2}\|\nabla f(x_k)\|^2 - \alpha_k(1-\alpha_kL)\langle \nabla f(x_k), z_k\rangle + \frac{\alpha_k^2LM^2}{2},
\end{align*} where the third inequality follows from Assumption \ref{asmpt:noise-bounded} and the choice $a \leq \frac{1}{L}$. Rearranging, summing up the first $t$ iterates and using the fact that the step-sizes are non-increasing, we have
\begin{equation}\label{eq:init-step}
    \frac{\alpha_t}{2}\sum_{k = 1}^t\|\nabla f(x_k)\|^2 \leq f(x_1) - f_\star + \sum_{k = 1}^t\alpha_k(\alpha_kL - 1)\langle \nabla f(x_k), z_k \rangle + \frac{LM^2}{2}\sum_{k = 1}^t\alpha_k^2.
\end{equation} Multiplying both sides in \eqref{eq:init-step} by $\frac{2}{\alpha_tt}$ and using the fact that $\min_{k \in [t]}\|\nabla f(x_k)\|^2 \leq \frac{1}{t}\sum_{k = 1}^t\|\nabla f(x_k)\|^2$, we get
\begin{equation}\label{eq:setup}
    \min_{k \in [t]}\|\nabla f(x_k)\|^2 \leq \frac{2}{\alpha_tt}\left(f(x_1) - f_\star + \sum_{k = 1}^t\alpha_k(\alpha_kL - 1)\langle \nabla f(x_k), z_k \rangle + \frac{LM^2}{2}\sum_{k = 1}^t\alpha_k^2\right).
\end{equation} Recall that we use the shorthand $F_t = \min_{k \in [t]}\|\nabla f(x_k)\|^2$, let $\Delta \coloneqq f(x_1) - f_\star$ and denote by $\lambda_t \coloneqq n_t\lambda$, where $\lambda \in \R$ and $\{n_t\}_{t \in \N}$ is a positive sequence to be specified later. We now want to bound the log MGF of $F_t$, denoted by $\Lambda_t$, i.e., $\Lambda_t(\lambda) \coloneqq \log\E\left[\exp\left(\lambda F_t\right)\right]$, for any $\lambda \in \R$. First, note that for any $\lambda < 0$, we have $\Lambda_t(\lambda_t) \leq 0$  (since $n_t,F_t \geq 0$), hence  
\begin{equation}\label{eq:key-ineq1}
    \limsup_{t \rightarrow \infty}\frac{\Lambda(n_t\lambda)}{n_t} \leq 0.
\end{equation} Next, consider any $\lambda \geq 0$. For ease of notation, let $b_t \coloneqq \frac{2\lambda_t}{\alpha_tt}$. From the definition of $\lambda_t$ and \eqref{eq:setup}, we have 
\begin{align*}
    &\E\big[\exp\big(\lambda_t F_t\big) \big] \leq \E\bigg[\exp\bigg(b_t\bigg(\Delta + \sum_{k = 1}^t\alpha_k(\alpha_kL - 1)\langle \nabla f(x_k), z_k \rangle + \frac{LM^2}{2}\sum_{k = 1}^t\alpha_k^2 \bigg)\bigg)\bigg] \\ 
    &= \exp\bigg(b_t\bigg(\Delta + \frac{LM^2}{2}\sum_{k = 1}^t\alpha_k^2\bigg)\bigg)\E\bigg[\exp\bigg(b_t\sum_{k = 1}^t\alpha_k(\alpha_kL - 1)\langle \nabla f(x_k), z_k\rangle\bigg)\bigg].
\end{align*} We now proceed to bound the last term. For ease of notation, let $Z_k \coloneqq \alpha_k(\alpha_kL - 1)\langle \nabla f(x_k), z_k\rangle$ and $\E_k[\cdot] \coloneqq \E[ \cdot | \mathcal{F}_k]$. Using Lemma \ref{lm:sub-gauss}, we then get
\begin{align*}
    \E\bigg[\exp\bigg(b_t\sum_{k = 1}^tZ_k\bigg)\bigg] &= \E\bigg[\exp\bigg(b_t\sum_{k = 1}^{t-1}Z_k \bigg)\E_t\Big[\exp\big(b_t\alpha_t(\alpha_tL - 1)\langle \nabla f(x_t), z_t\rangle\big) \Big] \bigg] \\
    &\leq \E\bigg[\exp\bigg(b_t\sum_{k = 1}^{t-1}Z_k \bigg)\exp\bigg(\frac{3b_t^2\alpha_t^2(1-\alpha_tL)^2M^2\|\nabla f(x_t)\|^2}{4}\bigg)\bigg] \\
    &\leq \exp\bigg(\frac{3b_t^2\alpha_t^2M^2G^2}{4}\bigg)\E\bigg[\exp\bigg(b_t\sum_{k = 1}^{t-1}Z_k \bigg) \bigg] \\
    &\leq \ldots \leq \exp\bigg(\frac{3b_t^2M^2G^2}{4}\sum_{k = 1}^t\alpha_k^2\bigg),
\end{align*} where the second inequality follows from Assumption \ref{asmpt:cost} and the fact that $(1-\alpha_kL)^2 \leq 1$, for all $k \geq 1$. Combining everything, we get
\begin{equation*}
    \E[\exp(\lambda_tF_t)] \leq \exp\bigg(\frac{2\lambda_t\Delta}{\alpha_tt} + \frac{3\lambda_t^2M^2G^2}{\alpha_t^2t^2}\sum_{k = 1}^t\alpha_k^2 + \frac{2\lambda_tLM^2}{\alpha_tt}\sum_{k = 1}^t\alpha_k^2 \bigg).
\end{equation*} Taking the logarithm, dividing by $n_t$ and recalling that $\alpha_k = \frac{a}{\sqrt{k+1}}$, we have
\begin{equation}\label{eq:semi-final}
    \frac{\Lambda_t(\lambda_t)}{n_t} \leq \frac{2\sqrt{2}\lambda\Delta}{a\sqrt{t}} + \frac{6n_t\lambda^2M^2G^2}{t}\sum_{k = 1}^t\frac{1}{k+1} + \frac{2\sqrt{2}a\lambda LM^2}{\sqrt{t}}\sum_{k = 1}^t\frac{1}{k+1}.
\end{equation} Using the Darboux sum approximation, we then get 
\begin{equation*}
    \sum_{k = 1}^t\frac{1}{k+1} \leq \int_{0}^{t}\frac{1}{k+1}dk = \log(t+1).
\end{equation*} Plugging in \eqref{eq:semi-final}, it follows that
\begin{equation*}
    \frac{\Lambda_t(\lambda_t)}{n_t} \leq \frac{2\sqrt{2}\lambda\Delta}{a\sqrt{t}} + \frac{6n_t\lambda^2M^2G^2\log(t+1)}{t} + \frac{2\sqrt{2}a\lambda LM^2\log(t+1)}{\sqrt{t}}.
\end{equation*} Choosing $n_t = \frac{t}{\log(t)}$ and taking the $\limsup$, we finally get
\begin{equation}\label{eq:key-ineq2}
    \limsup_{t \rightarrow \infty}\frac{\Lambda_t(n_t\lambda)}{n_t} \leq 6\lambda^2M^2G^2.
\end{equation} Define the continuous function $\varphi: \R \mapsto [0,\infty)$, given by
\begin{equation*}
    \varphi(\lambda) = \begin{cases}
        6\lambda^2M^2G^2, & \lambda \geq 0 \\
        0, & \lambda < 0.
    \end{cases}
\end{equation*} From \eqref{eq:key-ineq1} and \eqref{eq:key-ineq2}, it follows that, for any $\lambda \in \R$
\begin{equation*}
    \limsup_{t \rightarrow \infty}\frac{\Lambda_t(n_t\lambda)}{n_t} \leq \varphi(\lambda) < +\infty.
\end{equation*} The proof is then completed by invoking Proposition \ref{prop:gartner-ellis} and noting that the Fenchel-Legendre transform of $\varphi$ is given by
\begin{equation*}
    \varphi^\star(x) = \begin{cases}
        \frac{x^2}{24M^2G^2}, & x \geq 0 \\
        +\infty, & x < 0.
    \end{cases}
\end{equation*}

\section{Proof of Theorem \ref{thm:main-non-conv-clip}}\label{app:clip}

Recall that the clipping bias is decomposed as $\widetilde{g}_k - \nabla f(x_k) = \theta_k^u + \theta_k^b$, where $\theta_k^u = \widetilde{g}_k - \E_k[\widetilde{g}_k]$ and $\theta_k^b = \E_k[\widetilde{g}_k] - \nabla f(x_k)$. Using $L$-smoothness and the \csgd update rule \eqref{eq:sgd}, with $\Psi_t(x) = \min\Big\{1,\frac{\gamma_t}{\|x\|}\Big\}x$, we then have
\begin{align*}
    f(x_{k+1}) &\leq f(x_k) - \alpha_k\langle \nabla f(x_k), \widetilde{g}_k \rangle + \frac{\alpha_k^2L}{2}\|\widetilde{g}_k\|^2 \\
    &\leq f(x_k) - \alpha_k\|\nabla f(x_k)\|^2 - \alpha_k\langle \nabla f(x_k), \theta_k^u + \theta_k^b\rangle + \frac{\alpha_k^2\gamma_k^2L}{2} \\
    &\leq f(x_k) - \alpha_k\|\nabla f(x_k)\|^2 - \alpha_k\langle \nabla f(x_k), \theta_k^u\rangle + \frac{\alpha_k}{2}\|\nabla f(x_k)\|^2 + \frac{\alpha_k}{2}\|\theta_k^b\|^2 + \frac{\alpha_k^2\gamma_k^2L}{2} \\
    &\leq f(x_k) - \frac{\alpha_k}{2}\|\nabla f(x_k)\|^2 - \alpha_k\langle \nabla f(x_k), \theta_k^u\rangle + 8\alpha_k\sigma^{2p}\gamma_k^{2(1-p)} + \frac{\alpha_k^2\gamma_k^2L}{2},
\end{align*} where the third inequality follows from $\langle a,b\rangle \leq \frac{\|a\|^2}{2} + \frac{\|b\|^2}{2}$, while in the fourth inequality we use Lemma \ref{lm:clip-bias-bound}. Rearranging, summing up the first $t$ iterates and using the fact that the step-sizes are non-increasing, we have
\begin{equation}\label{eq:init-step-clip}
    \frac{\alpha_t}{2}\sum_{k = 1}^t\|\nabla f(x_k)\|^2 \leq f(x_1) - f_\star - \sum_{k = 1}^t\alpha_k\langle \nabla f(x_k), \theta_k^u \rangle + \sum_{k = 1}^t\bigg(8\alpha_k\sigma^{2p}\gamma_k^{2(1-p)} + \frac{\alpha_k^2\gamma_k^2L}{2}\bigg).
\end{equation} Multiplying both sides in \eqref{eq:init-step-clip} by $\frac{2}{\alpha_tt}$ and using the fact that $F_t \leq \frac{1}{t}\sum_{k = 1}^t\|\nabla f(x_k)\|^2$, with $F_t = \min_{k \in [t]}\|\nabla f(x_k)\|^2$, we get
\begin{equation}\label{eq:setup-clip}
    F_t \leq \frac{2}{\alpha_tt}\left(f(x_1) - f_\star - \sum_{k = 1}^t\alpha_k\langle \nabla f(x_k), \theta_k^u \rangle + \sum_{k = 1}^t\bigg(8\alpha_k\sigma^{2p}\gamma_k^{2(1-p)} + \frac{\alpha_k^2\gamma_k^2L}{2}\bigg)\right).
\end{equation} Let $\Delta \coloneqq f(x_1) - f_\star$ and denote by $\lambda_t \coloneqq n_t\lambda$, where $\lambda \in \R$ and $\{n_t\}_{t \in \N}$ is a positive sequence to be specified later. We again want to bound the log MGF of $F_t$ evaluated at $\lambda_t$. First, if $\lambda < 0$, it follows that
$\E[\exp(\lambda_tF_t)] \leq 1,$ which readily implies 
\begin{equation}\label{eq:key-ineq1-clip}
    \limsup_{t \rightarrow \infty}\frac{\Lambda(n_t\lambda)}{n_t} \leq 0.
\end{equation} Next, let $\lambda \geq 0$ and denote by $b_t \coloneqq \frac{2\lambda_t}{\alpha_tt}$. From the definition of $\lambda_t$ and \eqref{eq:setup-clip}, we have 
\begin{align}\label{eq:mgf-bdd-clip}
    \E\big[\exp\big(\lambda_t F_t\big) \big] \leq \exp&\bigg(b_t\bigg(\Delta + \sum_{k = 1}^t\bigg(8\alpha_k\sigma^{2p}\gamma_k^{2(1-p)} + \frac{\alpha_k^2\gamma_k^2L}{2} \bigg)\bigg)\bigg) \nonumber \\
    &\times \E\bigg[\exp\bigg(-b_t\sum_{k = 1}^t\alpha_k\langle \nabla f(x_k), \theta_k^u \rangle\bigg)\bigg].
\end{align} To bound the last term, we successively apply Lemma \ref{lm:clip-bias-bound} and use Assumption \ref{asmpt:cost}, to get
\begin{align}\label{eq:noise-bdd-clip}
    \E\bigg[\exp\bigg(-b_t\sum_{k = 1}^t\alpha_k\langle \nabla f(x_k),\theta_k^u \rangle\bigg)\bigg] \leq \exp\bigg(3b_t^2G^2\sum_{k = 1}^t\alpha_k^2\gamma_k^2\bigg).
\end{align} Combining \eqref{eq:mgf-bdd-clip} and \eqref{eq:noise-bdd-clip}, we get
\begin{equation*}
    \E[\exp(\lambda_tF_t)] \leq \exp\Bigg(\frac{2\lambda_t}{\alpha_tt}\bigg(\Delta + \sum_{k = 1}^t\bigg(8\alpha_k\sigma^{2p}\gamma_k^{2(1-p)} + \frac{\alpha_k^2\gamma_k^2L}{2}\bigg)\bigg) + \frac{12\lambda_t^2G^2}{\alpha_t^2t^2}\sum_{k = 1}^t\alpha_k^2\gamma_k^2 \Bigg).
\end{equation*} Taking the logarithm and dividing by $n_t$, we have
\begin{equation}\label{eq:semi-final-clip}
    \frac{\Lambda_t(\lambda_t)}{n_t} \leq \frac{2\lambda}{\alpha_tt}\bigg(\Delta + \sum_{k = 1}^t\bigg(8\alpha_k\sigma^{2p}\gamma_k^{2(1-p)} + \frac{\alpha_k^2\gamma_k^2L}{2}\bigg)\bigg) + \frac{12n_t\lambda^2G^2}{\alpha_t^2t^2}\sum_{k = 1}^t\alpha_k^2\gamma_k^2.
\end{equation} We now consider two cases with respect to the noise moment condition. First, if $p \in (1,2)$, we set $\alpha_k = (k+1)^{-\frac{p}{3p-2}}$, $\gamma_k = 2G(k+1)^{\frac{2-p}{6p-4}}$ and note that $\alpha_tt \geq t^{\frac{2(p-1)}{3p-2}}/2$. Moreover, using the Darboux sum approximation, it follows that
\begin{align}\label{eq:darboux-1}
    \sum_{k = 1}^t\alpha_k\gamma_k^{2(1-p)} &= 4^{1-p}G^{2(1-p)}\sum_{k = 1}^t(k+1)^{\frac{-p + (2-p)(1-p)}{3p-2}} \leq 4^{1-p}G^{2(1-p)}\sum_{k = 1}^t(k+1)^{\frac{2-4p + p^2}{3p-2}} \nonumber \\
    &\leq 4^{1-p}G^{2(1-p)}\int_{1}^{t+1}k^{\frac{2-4p + p^2}{3p-2}}dk \leq C_1t^{\frac{p(p-1)}{3p-2}},
\end{align} where $C_1 = \frac{\sqrt{2}(3p-2)4^{1-p}G^{2(1-p)}}{p(p-1)}$, as well as
\begin{equation}\label{eq:darboux-2}
    \sum_{k = 1}^t\alpha_k^2\gamma_k^2 = 4G^2\sum_{k = 1}^t(k+1)^{\frac{-2p + 2 - p}{3p-2}} = 4G^2\sum_{k = 1}^t\frac{1}{k+1} \leq 4G^2\log(t+1). 
\end{equation} Plugging \eqref{eq:darboux-1} and \eqref{eq:darboux-2} in \eqref{eq:semi-final-clip} and choosing $n_t = \frac{t^{\frac{4(p-1)}{3p-2}}}{\log(t)}$, it then follows that
\begin{align*}
    \frac{\Lambda_t(\lambda_t)}{n_t} &\leq \frac{4\lambda\Delta}{t^{\frac{2(p-1)}{3p-2}}} + \frac{32C_1\sigma^{2p}\lambda}{t^{\frac{(p-1)(2-p)}{3p-2}}} + \frac{8\lambda LG^2\log(t+1)}{t^{\frac{2(p-1)}{3p-2}}} + \frac{192\lambda^2G^4\log(t+1)}{\log(t)}.
\end{align*} Noting that $(p-1)(2-p) > 0$ for any $p \in (1,2)$ and taking the $\limsup$, we get
\begin{equation}\label{eq:key-ineq3-clip}
    \limsup_{t \rightarrow \infty}\frac{\Lambda_t(n_t\lambda)}{n_t} \leq 192\lambda^2G^4.
\end{equation} Next, if $p = 2$, we set $\alpha_k = \frac{1}{\sqrt{k+1}}$, $\gamma_k = 2G\sqrt{\log(k+1)}$ and note that $\alpha_tt \geq \sqrt{\frac{t}{2}}$. Moreover, using Darboux sum approximation, it follows that 
\begin{align}\label{eq:darboux-3}
    \sum_{k = 1}^t\frac{\alpha_k}{\gamma_k^2} &= \frac{1}{4G^2}\sum_{k = 1}^t\frac{1}{\log(k+1)\sqrt{k+1}} = C_2 + \frac{1}{4G^2}\sum_{k = 4}^t\frac{1}{\log(k+1)\sqrt{k+1}} \nonumber \\
    &\leq C_2 + \frac{1}{4G^2}\int_{4}^{t+1}\frac{dk}{\log(k)\sqrt{k}} = C_2 + \frac{1}{4G^2}\int_{2}^{\sqrt{t+1}}\frac{ds}{\log(s)} \leq C_2 + \frac{\text{li}(\sqrt{t+1})}{4G^2},
\end{align} where $C_2 = \frac{1}{4G^2}\sum_{k = 1}^3\frac{1}{\log(k+1)\sqrt{k+1}}$, the fourth (in)equality follows by introducing the substitution $s = \sqrt{k}$, while $\text{li}: (0,\infty) \mapsto \R$ is the logarithmic integral function, given by $\text{li}(x) = \int_0^x\frac{dt}{\log(t)}$. Moreover, we have
\begin{equation}\label{eq:darboux-4}
    \sum_{k = 1}^t\alpha_k^2\gamma_k^2 = 4G^2\sum_{k = 1}^t\frac{\log(k+1)}{k+1} \leq 4G^2\log(t+1)\sum_{k = 1}^t\frac{1}{k+1} \leq 4G^2\log^2(t+1).
\end{equation} Plugging \eqref{eq:darboux-3} and \eqref{eq:darboux-4} in \eqref{eq:semi-final-clip}, using the fact that $\text{li}(x) = \mathcal{O}\big(\frac{x}{\log(x)}\big)$ for $x$ sufficiently large, see, e.g., Chapter 6 in \cite{NIST:DLMF}, and choosing $n_t = \frac{t}{\log^2(t)}$, it then follows that for $t$ sufficiently large
\begin{align*}
    \frac{\Lambda_t(\lambda_t)}{n_t} \leq& \frac{2\sqrt{2}\lambda\Delta}{\sqrt{t}} + \frac{16\sqrt{2}\sigma^4\lambda}{\sqrt{t}}\bigg(C_2 + \mathcal{O}\bigg(\frac{\sqrt{t+1}}{\log(t+1)}\bigg) \bigg) \\
    &+ \frac{4\sqrt{2}\lambda LG^2\log^2(t+1)}{\sqrt{t}} + \frac{96\lambda^2G^4\log^2(t+1)}{\log^2(t)}.
\end{align*} Taking the $\limsup$, we finally get
\begin{equation}\label{eq:key-ineq2-clip}
    \limsup_{t \rightarrow \infty}\frac{\Lambda_t(n_t\lambda)}{n_t} \leq 96\lambda^2G^4.
\end{equation} Define the function $\varphi_p: [0,\infty) \mapsto [0,\infty)$, given by
\begin{equation*}
    \varphi_p(\lambda) = \begin{cases}
        96\lambda^2G^4, & p = 2 \\
        192\lambda^2G^4, & p \in (1,2),
    \end{cases}
\end{equation*} and consider the continuous function $\varphi: \R \times (1,2] \mapsto [0,\infty)$
\begin{equation*}
    \varphi(\lambda,p) = \begin{cases}
        \varphi_p(\lambda), & \lambda \geq 0 \\
        0, & \lambda < 0.
    \end{cases}
\end{equation*} From \eqref{eq:key-ineq1-clip}, \eqref{eq:key-ineq3-clip} and \eqref{eq:key-ineq2-clip}, it follows that, for any $\lambda \in \R$ and $p \in (1,2]$
\begin{equation*}
    \limsup_{t \rightarrow \infty}\frac{\Lambda_t(n_t\lambda)}{n_t} \leq \varphi(\lambda,p) < +\infty.
\end{equation*} The proof is complete by invoking Proposition \ref{prop:gartner-ellis} and noting that the Fenchel-Legendre transform of $\varphi_p$ is given by
\begin{equation*}
    \varphi^\star_p(x) = \begin{cases}
        \frac{x^2}{768G^4}, & p \in (1,2) \\
        \frac{x^2}{384G^4}, & p = 2.
    \end{cases}
\end{equation*}

\section{Proof of Theorem \ref{thm:low-bdd}}\label{app:low-bdd}

To prove Theorem \ref{thm:low-bdd}, we construct a specific instance of \eqref{eq:problem}, which obeys our assumptions and show that there exists a problem related constant $b > 0$ and global constants $a_1,a_2 > 0$, such that, if the iterates $\{x_t\}_{t \in \N}$ are generated by either \sgd or \csgd, then for any $\epsilon \in (0,b)$ and $t \geq 1$
\begin{equation}\label{eq:lower-bdd}
    \Prob(F_t > \epsilon) \geq a_1e^{-a_2t}.
\end{equation} To that end, consider the Huber cost, e.g., \cite{huber_loss}, which is given by
\begin{equation}\label{eq:huber}
    f(x) = \begin{cases}
        \frac{\|x\|^2}{2}, & \|x\| \leq G \\
        G\|x\| - \frac{G^2}{2}, & \|x\| > G
    \end{cases},
\end{equation} for a user-specified threshold $G > 0$. It can be readily seen from the definition in \eqref{eq:huber} that $f$ is lower bounded by zero and continuously differentiable, with the gradient given by
\begin{equation}\label{eq:huber-grad}
    \nabla f(x) = \begin{cases}
        x, & \|x\| \leq G \\
        \frac{Gx}{\|x\|}, & \|x\| > G.
    \end{cases}
\end{equation} It is easy to see from \eqref{eq:huber-grad} that $\|\nabla f(x)\| \leq G$, for all $x \in \R^d$. Moreover, it can be shown that the gradient of $f$ is Lipschitz continuous, with constant $L = 2$, see, e.g., Lemma B.2 in \cite{pmlr-v162-armacki22a}. As such, the Huber cost given in \eqref{eq:huber} satisfies Assumption \ref{asmpt:cost}. Next, let $x_1 \in \R^d$ be a deterministic initialization (hence satisfying Assumption \ref{asmpt:init}), such that 
\begin{equation}\label{eq:init-cond}
  0 < \|x_1\| \leq G.
\end{equation} Finally, consider a \sfo which, when queried with $x_t$, returns $g_t = \nabla f(x_t) + z_t$, where at each time $t \geq 1$, the noise instances $z_t$ are independent and identically distributed, according to the rule
\begin{equation}\label{eq:noise}
    z_t = \begin{cases}
        x_1, & \text{with probability } \frac{1}{2} \\
        -x_1, & \text{with probability } \frac{1}{2}
    \end{cases}.
\end{equation} By definition, the noise is zero-mean. Moreover, using \eqref{eq:init-cond}, it follows that $\|z_t\| = \|x_1\| \leq G$ almost surely, for all $t \geq 1$, hence the noise satisfies Assumption \ref{asmpt:noise-bounded}. Consider first the vanilla \sgd method, with step-size $\alpha_t = \frac{1}{L\sqrt{t+1}} = \frac{1}{2\sqrt{t+1}}$. From \eqref{eq:huber-grad}-\eqref{eq:noise}, it can be readily verified that
\begin{equation*}
    x_2 = x_1 - \alpha_1g_1 = (1-\alpha_1)x_1 - \alpha_1z_1 = \begin{cases}
        x_1, & \text{with probability } \frac{1}{2} \\
        \Big(1 - \frac{1}{\sqrt{2}}\Big)x_1, & \text{with probability } \frac{1}{2}
    \end{cases}.
\end{equation*} Repeating the argument, using the independence of noise, \eqref{eq:noise}, the \sgd update and Bayes's rule, it can be inferred that, for any $t \geq 1$
\begin{equation}\label{eq:prob-identical}
    \Prob(x_t = x_{t-1} = \ldots = x_1) = 2^{-t+1}.
\end{equation} Denote by $B_t$ the event $B_t \coloneqq \{\omega: x_t(\omega) = \ldots = x_2(\omega) = x_1\}$ and note that $\Prob(B_t) \stackrel{\eqref{eq:prob-identical}}{=} 2^{1-t}$. Next, using \eqref{eq:huber-grad}-\eqref{eq:init-cond}, we have, for any $t \geq 1$, conditioned on $B_t$
\begin{equation}\label{eq:identity}
    F_t = \min_{k \in [t]}\|\nabla f(x_k)\|^2 = \|\nabla f(x_1)\|^2 = \|x_1\|^2.  
\end{equation} Finally, using \eqref{eq:prob-identical} and \eqref{eq:identity}, it follows that, for any $\epsilon \in (0,\|x_1\|^2)$
\begin{align*}
    \Prob(F_t > \epsilon) \geq \Prob(\{\omega: F_t(\omega) > \epsilon\} \cap B_t) &= \Prob(F_t > \epsilon \: \vert \: B_t) \Prob(B_t) \\ 
    &= 2^{-t+1}\underbrace{\Prob(\|x_1\|^2 > \epsilon)}_{=1} = 2e^{-t\ln2}, 
\end{align*} which proves \eqref{eq:lower-bdd}, with $b = \|x_1\|^2$, $a_1 = 2$ and $a_2 = \ln 2$. Consider next the \csgd method, using the same step-size $\alpha_t = \frac{1}{2\sqrt{t+1}}$, with any clipping threshold satisfying $\gamma_t \geq 2G$, for all $t \geq 1$ (note that this is clearly satisfied for the clipping threshold specified in \eqref{eq:clipping-radius}). We now want to show, by induction, that for every $t \geq 1$
\begin{equation}\label{eq:clip-hypothesis}
    \|x_t\| \leq G,
\end{equation} which, together with the choice of oracle, cost and clipping threshold, would imply
\begin{equation}\label{eq:no-clip}
    \|g_t\| \leq \|\nabla f(x_t)\| + \|z_t\| \leq \|x_t\| + G \leq 2G \leq \gamma_t.
\end{equation} Equations \eqref{eq:clip-hypothesis} and \eqref{eq:no-clip} indicate that clipping is never performed and \csgd reverts to vanilla \sgd, implying that the same lower bound shown above holds for \csgd. Condition \eqref{eq:clip-hypothesis} clearly holds for the case $t = 1$, by the design of the initialization in \eqref{eq:init-cond}. Therefore, assume that \eqref{eq:clip-hypothesis} holds for some $t > 1$. As shown in \eqref{eq:no-clip}, we then know that clipping will not be performed, i.e., $\widetilde{g}_t = g_t$, and clipping reverts to \sgd in iteration $t+1$, which implies that
\begin{align*}
    \|x_{t+1}\| &= \|x_t - \alpha_t\widetilde{g}_t\| = \|x_t - \alpha_tg_t\| \stackrel{(i)}{=} \|(1-\alpha_t)x_t - \alpha_tz_t\| \\
    &\leq (1-\alpha_t)\|x_t\| + \alpha_t\|z_t\| \leq (1-\alpha_t)G + \alpha_t\|x_1\| \stackrel{(ii)}{\leq} G,
\end{align*} where $(i)$ follows from \eqref{eq:huber-grad} and the induction hypothesis \eqref{eq:clip-hypothesis}, while $(ii)$ follows from \eqref{eq:init-cond}. Therefore, we have shown that \eqref{eq:clip-hypothesis} holds for every $t \geq 1$. Therefore, \csgd reverts to \sgd and the lower bound established above holds for \csgd as well, completing the proof.

\section{On the Metric}\label{app:metric}

As discussed in the main body, while we use the metric $F_t = \min_{k \in [t]}\|\nabla f(x_k)\|^2$, our results continue to hold for the metric $A_t = \frac{1}{t}\sum_{k = 1}^t\|\nabla f(x_k)\|^2$, which is stronger, seeing that $F_t \leq A_t$. To see why this is the case, note that in both the proof of Theorem \ref{thm:main-non-conv} (recall \eqref{eq:init-step}-\eqref{eq:setup}) and that of Theorem \ref{thm:main-non-conv-clip} (recall \eqref{eq:init-step-clip}-\eqref{eq:setup-clip}), we start with a bound on $A_t$ (in \eqref{eq:init-step} and \eqref{eq:init-step-clip}, respectively) and use the fact that $F_t \leq A_t$ to switch to the desired metric $F_t$ (in \eqref{eq:setup} and \eqref{eq:setup-clip}, respectively). It can then be readily seen that skipping the inequality $F_t \leq A_t$, while using the exact same steps in the respective proofs, we would get the same results on the more general metric $A_t$. Similarly, the results of Theorem \ref{thm:low-bdd} hold for $A_t$, which can be seen by noting that, conditioned on the event $B_t$ defined in the proof of Theorem \ref{thm:low-bdd}, we have $A_t = \frac{1}{t}\sum_{k = 1}^t\|\nabla f(x_k)\|^2 = \|\nabla f(x_1)\|^2 = \|x_1\|^2$, matching the value of $F_t$ in \eqref{eq:identity}. As mentioned in the main body, the reason for presenting our results in terms of the metric $F_t$ instead of $A_t$ stems from the fact that tail bounds on $F_t$ have a more direct interpretation, as $F_t > \epsilon^2$ readily implies that an $\epsilon$-stationary point has not been reached in any of the first $t$ iterations. On the other hand, $A_t > \epsilon^2$ does not necessarily imply that an $\epsilon$-stationary point has not been reached, as it is possible that $A_t > \epsilon^2 \geq F_t$, hence an $\epsilon$-stationary point can be reached, while having $A_t > \epsilon^2$.

\section{Clipping With a General Threshold}\label{app:clip-relax}

In this section, we discuss how our results for \csgd can be extended when using a clipping threshold which does not requiring knowledge of the gradient bound $G$. In particular, we consider a clipping threshold of the form
\begin{equation}\label{eq:clipping-radius-general}
    \gamma_t = \begin{cases}
            C(t+1)^{\frac{2-p}{6p-4}}, & p \in (1,2) \\
            C\sqrt{\log(t+1)}, & p = 2,
        \end{cases}
\end{equation} where $C > 0$ is any user-specified constant. We now want to establish a counterpart of Lemma \ref{lm:clip-bias-bound}, for the general clipping threshold in \eqref{eq:clipping-radius-general}. To that end, we have the following result.

\begin{lemma}\label{lm:clip-bias-bound-general}
    Let Assumptions \ref{asmpt:cost} and \ref{asmpt:noise-heavy} hold and let the clipping threshold be chosen as in \eqref{eq:clipping-radius-general}. Then, the following are true.
    \begin{enumerate}
        \item For all $t \geq B_p$, we have $\|\theta_t^b\| \leq 4\sigma^p\gamma_t^{1-p}$, where $B_p = \begin{cases}
            \Big(\frac{2G}{C}\Big)^{\frac{6p-4}{2-p}}, & p \in (1,2) \\
            2^{\nicefrac{4G^2}{C^2}} - 1, & p = 2
        \end{cases}$.

        \item For all $t \geq 1$ and any $\mathcal{F}_t$-measurable $x \in \R^d$, we have $\E\big[\exp\big(\langle x, \theta_t^u\rangle \big) \: \vert \: \mathcal{F}_t \big] \leq \exp\big(3\gamma_t^2\|x\|^2\big)$.
    \end{enumerate}
\end{lemma}

\begin{proof}
    To prove the first part, note that from Assumption \ref{asmpt:cost}, the choice of clipping threshold in \eqref{eq:clipping-radius-general} and the definition of $B_p$, we have, for any $t \geq B_p$
    \begin{equation*}
        \|\nabla f(x_t)\| \leq G \leq \frac{\gamma_t}{2}. 
    \end{equation*} The claim now readily follows by applying Proposition \ref{prop:sadiev-technical}. The second part is proved in the same way as in the proof of Lemma \ref{lm:clip-bias-bound}.
\end{proof}

Lemma \ref{lm:clip-bias-bound-general} shows that the bound on the unbiased component established in Lemma \ref{lm:clip-bias-bound} remains valid for all times $t$, while the bound on the biased component also becomes active, after a certain number of iterations. We then have the following result.

\begin{theorem}
    Let Assumptions \ref{asmpt:init}, \ref{asmpt:cost} and \ref{asmpt:noise-heavy} hold and let $\{x_t\}_{t \in \N}$ be the sequence generated by \csgd using the step-size $\alpha_t = (t+1)^{-\frac{p}{3p-2}}$ and clipping threshold $\gamma_t$ given in \eqref{eq:clipping-radius-general}. Then the sequence $\{F_t\}_{t \in \N}$ satisfies an LDP upper bound, with the following decay rate and rate function.
    \begin{enumerate}
        \item If $p \in (1,2)$, the decay rate is $n_t = \frac{t^{\frac{4(p-1)}{3p-2}}}{\log(t)}$, with rate function given by $I_c(x) = \begin{cases}
            \frac{x^2}{192C^2G^2}, & x \geq 0 \\
            +\infty, & x < 0.
        \end{cases}$

        \item If $p = 2$, the decay rate is $n_t = \frac{t}{\log^2(t)}$, with rate function given by $I_c(x) = \begin{cases}
            \frac{x^2}{96C^2G^2}, & x \geq 0 \\
            +\infty, & x < 0.
        \end{cases}$
    \end{enumerate}
\end{theorem}

\begin{proof}
    Recall that we showed the following inequality in Appendix \ref{app:clip}, for all $k \geq 1$ 
    \begin{align*}
    f(x_{k+1}) &\leq f(x_k) - \frac{\alpha_k}{2}\|\nabla f(x_k)\|^2 - \alpha_k\langle \nabla f(x_k), \theta_k^u\rangle + \frac{\alpha_k}{2}\|\theta_k^b\|^2 + \frac{\alpha_k^2\gamma_k^2L}{2}.
    \end{align*} Rearranging, summing up the first $t$ iterations and using the fact that the step-sizes are non-increasing, we have
    \begin{equation}\label{eq:init-step-clip-general}
        \frac{\alpha_t}{2}\sum_{k = 1}^t\|\nabla f(x_k)\|^2 \leq f(x_1) - f_\star - \sum_{k = 1}^t\alpha_k\langle \nabla f(x_k), \theta_k^u \rangle + \frac{1}{2}\sum_{k = 1}^t\alpha_k\|\theta_k^b\|^2 + \sum_{k = 1}^t\frac{\alpha_k^2\gamma_k^2L}{2}.
    \end{equation} We can see that the only difference between \eqref{eq:init-step-clip-general} and the corresponding equation \eqref{eq:init-step-clip} in the proof of Theorem \ref{thm:main-non-conv-clip} is the presence of $\sum_{k = 1}^t\alpha_k\|\theta_k^b\|^2$. Therefore, we our aim is to bound this expression. To that end, let $t \geq B_p$ (this is fine, since we are interested in the limit behaviour $t \rightarrow \infty)$, and notice that
    \begin{equation*}
        \sum_{k = 1}^t\alpha_k\|\theta_k^b\|^2 = \sum_{k = 1}^{B_p}\alpha_k\|\theta_k^b\|^2 + \sum_{k = B_p}^t\alpha_k\|\theta_k^b\|^2 \leq \sum_{k = 1}^{B_p}\alpha_k\|\theta_k^b\|^2 + 16\sum_{k = B_p}^t\alpha_k\sigma^{2p}\gamma_k^{2(1-p)},
    \end{equation*} where in the last inequality we used Lemma \ref{lm:clip-bias-bound-general}. What is left now is to show that the remaining term $\sum_{k = 1}^{B_p}\alpha_k\|\theta_k^b\|^2$ stays bounded. Recalling that $\theta_k^b = \E\big[\widetilde{g}_k \: \vert \: \mathcal{F}_k \big] - \nabla f(x_k)$, where $\widetilde{g}_k$ is the clipped stochastic gradient, and using the fact that $\alpha_k \leq 1$, we get
    \begin{equation*}
        \sum_{k = 1}^{B_p}\alpha_k\|\theta_k^b\|^2 \leq 2\sum_{k = 1}^{B_p}\big(\|\E\big[\widetilde{g}_k \: \vert \: \mathcal{F}_k \big] \|^2 + \|\nabla f(x_k)\|^2\big) \leq 2\sum_{k = 1}^{B_p}\big(\gamma_k^2 + G^2) \leq 10G^2B_p,
    \end{equation*} where the last inequality follows by noting that  $\gamma_k \leq 2G$, for all $k \leq B_p$. The rest of the proof now follows the same steps as in Theorem \ref{thm:main-non-conv-clip}, with $2G$ replaced by $C$, and is omitted, for brevity.
\end{proof}

\end{document}